\documentclass{article}
\usepackage[utf8]{inputenc} %
\usepackage[T1]{fontenc}    %
\usepackage{multirow}
\usepackage{caption}
\usepackage{subcaption}
\usepackage[table,svgnames]{xcolor}
\usepackage{enumitem}
\usepackage{amsthm}
\usepackage{amsmath,amssymb}
\usepackage{wrapfig}
\usepackage[colorlinks=true, allcolors=blue]{hyperref}
\usepackage[capitalise]{cleveref}

\newcommand{\ours}{\textsc{TabFlex}} %
\newcommand{\tabpfn}{\textsc{TabPFN}}
\newtheorem{theorem}{Theorem}

\newtheorem{lemma}[theorem]{Lemma}

\definecolor{pinegreen}{rgb}{0.0, 0.47, 0.44}
\definecolor{cornellred}{rgb}{0.7, 0.11, 0.11}
\definecolor{cadmiumgreen}{rgb}{0.0, 0.42, 0.24}
\definecolor{royalblue}{rgb}{0.0, 0.14, 0.4}
\definecolor{spirodiscoball}{rgb}{0.06, 0.75, 0.99}
\definecolor{mylightblue}{rgb}{0.85, 0.90, 0.94}
\definecolor{kaistblue}{RGB}{20,135,200}
\definecolor{auburn}{RGB}{166,38,57}

\hypersetup{
  urlcolor = cadmiumgreen,
  linkcolor = cornellred,
  citecolor  = cadmiumgreen,
  colorlinks = true,
}

\newcommand{\tf}{TF}
\newcommand{\nontf}{Non-TF NN}

\usepackage[textsize=tiny]{todonotes}

\def\A{{\bm{A}}}
\def\dA{\overline{\A}}
\def\B{{\bm{B}}}
\def\dB{\overline{\B}}
\def\C{{\bm{C}}}
\def\h{\bm{h}}
\def\dK{\overline{\mK}}
\def\dimh{H}

\usepackage[ruled,vlined,linesnumbered]{algorithm2e}

\usepackage{tikz}
\usetikzlibrary{positioning,fit,arrows.meta,decorations.pathreplacing,calligraphy}

\definecolor{mytrain}{HTML}{d8e2dc}
\definecolor{mytest}{HTML}{ffcad4}
\definecolor{myattn}{HTML}{ffe5d9}

\definecolor{mycolor1}{HTML}{f6f6f4}
\definecolor{mycolor2}{HTML}{ebe8e3}
\definecolor{mycolor3}{HTML}{f9f6f1}

\newcommand{\fe}[1]{\phi\left(#1\right)}
\usepackage{comment}
\usepackage{autonum}

\usepackage{tocloft} 
\usepackage{titletoc} 

\newcounter{temp_theorem_counter}   
\usepackage{listings}
    {}               %

\usepackage{amsmath,amsfonts,bm}

\def\eqref#1{(\ref{#1})}

\def\set#1{\{#1\}}
\def\1{\bm{1}}

\def\va{{\bm{a}}}

\def\vh{{\bm{h}}}

\def\vk{{\bm{k}}}

\def\vq{{\bm{q}}}

\def\vv{{\bm{v}}}

\def\mK{{\bm{K}}}

\def\mM{{\bm{M}}}

\def\mO{{\bm{O}}}

\def\mQ{{\bm{Q}}}

\def\mS{{\bm{S}}}

\def\mV{{\bm{V}}}

\DeclareMathAlphabet{\mathsfit}{\encodingdefault}{\sfdefault}{m}{sl}
\SetMathAlphabet{\mathsfit}{bold}{\encodingdefault}{\sfdefault}{bx}{n}

\def\gD{{\mathcal{D}}}

\def\sN{{\mathbb{N}}}

\def\sR{{\mathbb{R}}}

\usepackage{microtype}
\usepackage{booktabs} %
\usepackage{graphicx}

\usepackage{hyperref}

\usepackage[accepted]{icml2025}

\usepackage[textsize=tiny]{todonotes}

\icmltitlerunning{\ours{}: Scaling Tabular Learning to Millions with Linear Attention}

\begin{document}

\twocolumn[
\icmltitle{\ours{}: Scaling Tabular Learning to Millions with Linear Attention}

\icmlsetsymbol{equal}{*}

\begin{icmlauthorlist}
\icmlauthor{Yuchen Zeng }{equal,intern,wisc}
\icmlauthor{Tuan Dinh}{equal,ucsf}
\icmlauthor{Wonjun Kang}{furiosa,snu}
\icmlauthor{Andreas C. M\"ueller}{microsoft}
\end{icmlauthorlist}

\icmlaffiliation{wisc}{University of Wisconsin-Madison}
\icmlaffiliation{ucsf}{University of California San Francisco}
\icmlaffiliation{furiosa}{Furiosa AI}
\icmlaffiliation{snu}{Seoul National University}
\icmlaffiliation{microsoft}{Gray System Lab, Microsoft}
\icmlaffiliation{intern}{Work done during an internship at the Gray Systems Lab, Microsoft}
\icmlcorrespondingauthor{Andreas C. M\"ueller}{amueller@microsoft.com}

\icmlkeywords{Machine Learning, ICML}

\vskip 0.3in
]

\printAffiliationsAndNotice{\icmlEqualContribution} %

\begin{abstract} 
Leveraging the in-context learning (ICL) capability of Large Language Models (LLMs) for tabular classification has gained significant attention for its training-free adaptability across diverse datasets.
Recent advances, such as \tabpfn{}, excel in tabular small-scale datasets but struggle to scale for large and complex datasets. 
Our work enhances the efficiency and scalability of \tabpfn{} for larger datasets by incorporating linear attention mechanisms as a scalable alternative to complexity-quadratic self-attention.
Our model, \ours{}, efficiently handles tabular datasets with thousands of features and hundreds of classes, seamlessly scaling to millions of samples.
For instance, \ours{} processes the \texttt{poker-hand} dataset with more than a million samples in just 5 seconds.  
Our extensive evaluations demonstrate that \ours{} can achieve over a $2\times$ speedup compared to \tabpfn{} and a $1.5\times$ speedup over XGBoost, outperforming 25 tested baselines in terms of efficiency across a diverse range of datasets. 
Furthermore, \ours{} remains highly effective in large-scale datasets, delivering strong performance with significantly reduced computational costs, especially when combined with data-efficient techniques such as dimensionality reduction and data sampling. 
\end{abstract}
\vspace{-.1in}
\section{Introduction}
Enhancing the applicability of the Transformer architecture~\citep{vaswani2017attention} for diverse data modalities beyond textual data and non-language tasks~\citep{gpt4,brown2020gpt3,qwen,llama3} has achieved remarkable success~\cite{gemini}, from vision~\citep{qwenvl}, audio~\citep{qwenaudio,qwen2audio} to bio-signals~\cite{wan2023eegformer} and protein sequences~\cite{rives2019biological,hayes2024simulating}.
Tabular data, as one of the most fundamental and critical data types in real-world applications -- including recommendation systems~\citep{zhang2019deep}, finance~\citep{arun2016loan}, and medicine~\citep{johnson2016mimic} has attracted a great deal of attention and attempts to explore the potential of Transformer-based models, particularly for tabular classification~\citep{arik2021tabnet,hollmann2023tabpfn,huang2020tabtransformer,dinh2022lift,gorishniy2021revisiting}.
For example, the FT transformer~\citep{gorishniy2021revisiting} converts each sample into a sequence of embeddings to use the transformer to make predictions. 
TabTransformer~\citep{huang2020tabtransformer} learns embeddings for categorical features, concatenating them with continuous features. 
On the other hand, LIFT~\citep{dinh2022lift} converts tabular data combined with feature names and task descriptions into textual sentences as input to LLMs.
In particular, compared to traditional methods for tabular data such as gradient-boosted trees~\cite{friedman2001greedy}, these transformer-based methods often suffer from high latency overhead for training and inference, primarily due to their larger model sizes.

The recent \tabpfn{}~\citep{hollmann2023tabpfn} addresses the latency limitations of Transformer-based methods by utilizing the in-context learning (ICL) capability~\citep{brown2020gpt3} of LLMs for directly learning a new task from examples without parameter updates, attaining superior efficiency and performance on small-scale datasets.
In particular, \tabpfn{} incorporates all training and testing samples into a single prompt and classifies the testing samples in one forward pass, making it highly efficient and effective on simple and small tabular datasets.
However, \tabpfn{} faces challenges with complex datasets that require large sample sizes for effective learning, primarily due to the scalability limitations imposed by the quadratic complexity of the attention mechanism, introducing difficulties for both the scalable pretraining and inference processes.

In this work, we address the limitations of the scalability of \tabpfn{} and improve the effectiveness of Transformer-based methods for tabular classification. 
We first systematically analyze scalable alternatives to attention mechanisms, focusing on state-space models (SSMs) within the Mamba model~\citep{gu2024mamba} and linear attention~\citep{katharopoulos2020transformers}.
We find that \textbf{(Finding 1)} {the inherent causality of SSMs impedes ICL performance compared to non-causal mechanisms.}
In contrast, \textbf{(Finding 2)} {linear attention does not suffer from this limitation, maintaining comparable performance with improved computational efficiency.}
Thus, we develop \ours{} leveraging linear attention as the attention mechanism, comprising three sub-models where each is optimized for different scenarios and selected based on dataset characteristics (e.g., sample size). 
\ours{} efficiently handles tabular datasets with thousands of features and hundreds of classes, scaling to millions of samples.
Via the comprehensive evaluation on a diverse range of datasets, we find that \textbf{(Finding 3)} {\ours{} consistently achieves competitively high performance with impressive computational efficiency compared to 25 baselines, including \tabpfn{} and XGBoost.}
Notably, \ours{} perform highly on \texttt{poker-hand} dataset with $1$M+ samples in \textit{less than 5 seconds} and
attains high accuracies on image datasets such as MNIST~\citep{mnist}, Fashion-MNIST~\citep{fmnist}, and CIFAR-10~\cite{krizhevsky2009learning} in less than one second.
Furthermore, our ablation studies suggest that \ours{} can seamlessly incorporate data-efficient techniques such as dimensionality reduction and data sampling for more computation reduction.

\section{Related Works}
\label{sec:related_work}
\paragraph{Transformer-based approaches for tabular classification.}
The pioneering TabNet~\citep{arik2021tabnet} applies unsupervised pre-training on masked tabular datasets to infer missing features, enhancing the model's understanding of data and features before supervised learning on feature selection for the final decision boundary.
TabTransformer~\citep{huang2020tabtransformer} proposes handling categorical features by concatenating their contextual embeddings into numerical features. 
FT-Transformer~\citep{gorishniy2021revisiting} converts samples to embedding sequences using a feature tokenizer for the transformer. 
LIFT~\citep{dinh2022lift} converts each sample into a sentence using a predefined template incorporating the task description and feature names, as the natural input to apply ICL in LLM.
TabR~\citep{gorishniy2024tabr} proposes a retrieval-augmented model with a custom kNN-like component to retrieve and extract signals from the nearest neighbors. 
BiSHop~\citep{xu2024bishop} establishes interconnected directional learning modules to process data column-wise and row-wise for tabular learning.
XTab~\citep{zhu2023xtab} utilizes independent featurizers and federated learning to resolve inconsistent column types and quantities.

The widely adopted transformer-based approaches for tabular classification—\tabpfn{}~\citep{hollmann2023tabpfn} is trained offline on synthetic datasets derived from previous distributions to perform ICL, allowing efficient inference in small-scale tabular classification tasks.
However, it is limited to small tabular classification datasets.
To handle it, many concurrent variants are proposed.
MixturePFN~\citep{xu2025mixturepfn} improves scalability by routing new test samples to a pool of scalable prompters using Sparse Mixture of In-Context Prompters, while LoCalPFN~\citep{thomas2024retrieval} proposes retrieving a local subset of task-specific data for efficiently fine-tuning on. 
\citet{ma2024context} introduce in-context data distillation to optimize TabPFN’s context and remove the data size constraint.
TuneTable~\citep{feuer2024tunetables} scales \tabpfn{} to large datasets by performing a prefix tuning per dataset.
TabPFNv2~\citep{tabpfnv2} enhances TabPFN’s accuracy in low-data regimes (fewer than 10,000 samples), complementing our focus on speed and scalability.
Our method is also based on \tabpfn{}, extending its scalability to large datasets while maintaining and improving efficiency by simply replacing the softmax attention with linear attention.

\textbf{Attention mechanisms and scalable alternatives.~~}
As Transformers~\citep{vaswani2017attention} face the scaling challenge for long sequences due to the quadratic computational and memory complexity, scalable alternatives have been proposed~\citep{orvieto2023resurrecting,sun2023retentive}.
While RNNs provide efficient linear-time inference, they struggle with training efficiency and lack the parallelization capabilities of Transformer architectures. 
Linear attention~\citep{katharopoulos2020transformers} addresses both concerns by reformulating self-attention as a linear dot-product of kernel feature maps, reducing the computational complexity from quadratic to linear time. 
Furthermore, causal linear attention can be interpreted as a form of RNN, as it predicts based on a current token and a ``hidden state,'' summarizing information from the previous tokens. 
State-space models (SSMs) address RNNs' drawbacks by considering linear RNNs with novel algorithms for efficient training~\citep{gu2021combining,gu2022efficiently,gu2024mamba,dao2024transformers,peng2023rwkv,orvieto2023resurrecting,sun2023retentive}. 
\citet{dao2022flashattention} identified another bottleneck in attention mechanisms' speed stemming from the relatively slow access to high-bandwidth memory (HBM) in GPUs and proposed FlashAttention~\citep{dao2024flashattention,shah2024flashattention} to restructure attention computation to optimize the utilization of high-speed on-chip SRAM while minimizing access to slower HBM, enhancing the efficiency of GPU-based attention operations.

See Section~\ref{app:related_works} for an extended discussion of related works.

\def\myhz{1.8}
\def\myha{-3.6}
\def\myhb{-8}
\def\myhc{-12}

\def\myhx{-5.4}

\def\myhy{-7.2}

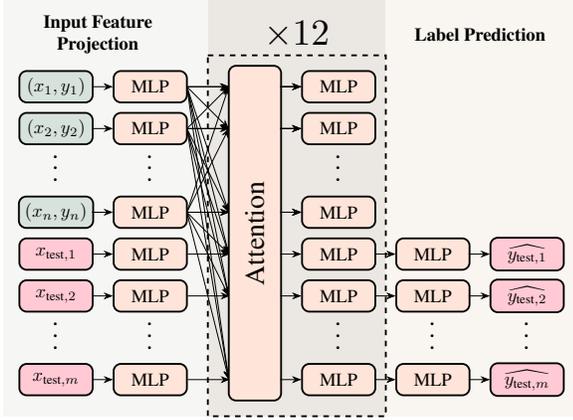
\begin{figure}
    \centering
\resizebox{.45\textwidth}{!}{
\begin{tikzpicture}[rotate=90,transform shape,xscale=-1,
    node distance = 0.5cm and 1cm,
    smallbox/.style = {draw, line width=1pt, minimum width=1.4cm, minimum height=0.6cm, rotate=90 ,rounded corners,xscale=-1,},
    attention/.style = {draw, line width=1pt, minimum width=6.4cm, minimum height=1cm, fill=myattn, rounded corners,xscale=-1},
    >=Stealth
]

\path[fill=mycolor1] (-1.7,2.8) rectangle ++(8,-3.9);
\node[xscale=-1,rotate=270, align=center] at (-1,1) {\textbf{Input Feature} \\ \textbf{Projection}};

\path[fill=mycolor2] (-1.7,-1.1) rectangle ++(8,-3.4);
\node[xscale=-1,rotate=270] at (-1,-2.8) {\scalebox{2}{$\times 12$}};

\path[fill=mycolor3] (-1.7,-4.5) rectangle ++(8,-3.7);
\node[xscale=-1,rotate=270, align=center] at (-1,-6.3) { \textbf{Label Prediction}};
\draw[draw=black,dashed,line width=1pt] (-0.6,-1.1) rectangle ++(6.9,-3.4);

\foreach \i [count=\xi from 0] in {1,2,n} {
    \ifnum\xi<2
        \node[smallbox,fill=mytrain] (x\i) at (\xi*0.8,\myhz) {\rotatebox{0}{$(x_{\i},y_{\i})$}};
    \else
        \node[smallbox,fill=mytrain] (x\i) at (0.8+\xi*0.8,\myhz) {\rotatebox{0}{$(x_{\i},y_{\i})$}};
        \node at (1.6,\myhz) {\Large$\cdots$};
    \fi
}
\foreach \i [count=\xi from 4] in {1,2,m} {
    \ifnum\xi<6
        \node[smallbox,fill=mytest] (xt\i) at (\xi*0.8,\myhz) {\rotatebox{0}{$x_{\text{test},\i}$}};
    \else
        \node[smallbox,fill=mytest] (xt\i) at (0.8+\xi*0.8,\myhz) {\rotatebox{0}{$x_{\text{test},\i}$}};
        \node at (4.8,\myhz) {\Large$\cdots$};
    \fi
}

\foreach \i [count=\xi from 0] in {1,2,n} {
    \ifnum\xi<2
        \node[smallbox,fill=myattn] (x\i) at (\xi*0.8,0) {\rotatebox{0}{MLP}};
    \else
        \node[smallbox,fill=myattn] (x\i) at (0.8+\xi*0.8,0) {\rotatebox{0}{MLP}};
        \node at (1.6,0) {\Large$\cdots$};
    \fi
}

\foreach \i [count=\xi from 4] in {1,2,m} {
    \ifnum\xi<6
        \node[smallbox,fill=myattn] (xt\i) at (\xi*0.8,0) {\rotatebox{0}{MLP}};
    \else
        \node[smallbox,fill=myattn] (xt\i) at (0.8+\xi*0.8,0) {\rotatebox{0}{MLP}};
        \node at (4.8,0) {\Large$\cdots$};
    \fi
}
\foreach \i in {1,2,n}
    \draw[<-] (x\i.west) -- ++(0,0.4);
\foreach \i in {1,2,m}
    \draw[<-] (xt\i.west) -- ++(0,0.4);

\foreach \j in {1,2,4}
    \draw[<-] (3.2,-1.5) -- (0.8*\j-0.8,-0.7);
\foreach \j in {1,2,4}
    \draw[<-] (4,-1.5) -- (0.8*\j-0.8,-0.7);
\foreach \j in {1,2,4}
    \draw[<-] (5.6,-1.5) -- (0.8*\j-0.8,-0.7);

\foreach \i in {1,2,n}
    \draw[->] (x\i.east) -- ++(0,-0.8);
\foreach \i in {1,2,4}
    \foreach \j in {1,2,4}
        \draw[->] (0.8*\i-0.8,-0.7) -- ++(0.8*\j-0.8*\i,-0.8);
\foreach \i in {1,2,m}
    \draw[->] (xt\i.east) -- ++(0,-0.8);

\node[attention] (att) at (2.8,-2) {\scalebox{1.5}{Attention}};

\foreach \i [count=\xi from 0] in {1,2,n} {
    \ifnum\xi<2
        \node[smallbox,fill=myattn] (x\i) at (\xi*0.8,\myha) {\rotatebox{0}{MLP}};
    \else
        \node[smallbox,fill=myattn] (x\i) at (0.8+\xi*0.8,\myha) {\rotatebox{0}{MLP}};
        \node at (1.6,\myha) {\Large$\cdots$};
    \fi
}

\foreach \i [count=\xi from 4] in {1,2,m} {
    \ifnum\xi<6
        \node[smallbox,fill=myattn] (xt\i) at (\xi*0.8,\myha) {\rotatebox{0}{MLP}};
    \else
        \node[smallbox,fill=myattn] (xt\i) at (0.8+\xi*0.8,\myha) {\rotatebox{0}{MLP}};
        \node at (4.8,\myha) {\Large$\cdots$};
    \fi
}

\foreach \i in {1,2,n}
    \draw[<-] (x\i.west) -- ++(0,0.4);
\foreach \i in {1,2,m}
    \draw[<-] (xt\i.west) -- ++(0,0.4);

\foreach \i [count=\xi from 4] in {1,2,m} {
    \ifnum\xi<6
        \node[smallbox,fill=myattn] (xt\i) at (\xi*0.8,\myhx) {\rotatebox{0}{MLP}};
    \else
        \node[smallbox,fill=myattn] (xt\i) at (0.8+\xi*0.8,\myhx) {\rotatebox{0}{MLP}};
        \node at (4.8,\myhx) {\Large$\cdots$};
    \fi
}

\foreach \i in {1,2,n}
    \draw[<-] (x\i.west) -- ++(0,0.4);
\foreach \i in {1,2,m}
    \draw[<-] (xt\i.west) -- ++(0,0.4);

\foreach \i [count=\xi from 4] in {1,2,m} {
    \ifnum\xi<6
        \node[smallbox,fill=mytest] (xt\i) at (\xi*0.8,\myhy) {\rotatebox{0}{$\widehat{y_{\text{test},\i}}$}};
    \else
        \node[smallbox,fill=mytest] (xt\i) at (0.8+\xi*0.8,\myhy) {\rotatebox{0}{$\widehat{y_{\text{test},\i}}$}};
        \node at (4.8,\myhy) {\Large$\cdots$};
    \fi
}

\foreach \i in {1,2,n}
    \draw[<-] (x\i.west) -- ++(0,0.4);
\foreach \i in {1,2,m}
    \draw[<-] (xt\i.west) -- ++(0,0.4);

\end{tikzpicture}}
    \captionsetup{skip=5pt}
    \vspace{-0.1in}
    \caption{
   \textbf{
    Illustration of \tabpfn{}'s for classifying the entire dataset in one forward pass. }
    In each layer, attention outputs for training sample positions attend to all other training samples, ensuring that predictions are invariant to the order of training samples.
    Conversely, attention outputs for test sample positions attend only to training samples, ensuring independent predictions for each test instance, unaffected by other test samples.
    The final classification for each test sample is derived by applying an MLP to the corresponding Transformer output at its respective position.
    }
    \label{fig:tabpfn_demo}
    \vspace{-.1in}
\end{figure}

\section{Preliminaries}
We elucidate key concepts of \tabpfn{} and two prominent scalable attention mechanisms (SSMs and linear attention).

\paragraph{Implementation of ICL in TabPFN~\citep{hollmann2023tabpfn}.} 
Fig.~\ref{fig:tabpfn_demo} illustrates the design of \tabpfn{}, where each sample is treated as a token, starting with training samples and followed by testing samples.
These samples are embedded (features $x$ and labels $y$ for training and only features $x$ for testing samples) with MLPs before being concatenated.
Outputs corresponding to training sample positions are computed by attending to all other training samples, while the outputs for test sample positions attend to the training samples — enabling each test prediction to leverage the full training set without being influenced by other test samples. 
Test predictions are generated by projecting the Transformer outputs at test positions into probability distributions. 
This implementation is functionally equivalent to standard ICL but significantly more efficient. 
Standard ICL requires $m$ (number of test samples) separate prompts, each containing all training samples and one test sample, necessitating $m$ prediction passes.
A notable feature of \tabpfn{} is the encoder with non-causal attention, allowing outputs within training sample positions to interact freely, rendering the order of training samples inconsequential.

\paragraph{State-Space Models (SSMs).}
The SSM framework is based on a continuous system that transforms a one-dimensional signal $x(t) \in \sR$ into $y(t) \in \sR$ through an intermediate $\dimh$-dimensional latent state $\h(t) \in \sR^\dimh$, as shown in \eqref{eq:rnn_continuous_ssm}. 
Here, $\B \in \sR^{\dimh \times 1}$ is the input transition vector and $\A \in \sR^{\dimh \times \dimh}$ is the state transition matrix. 
The latent state $\h(t)$ is then projected into the output $y(t)$ using the output mapping vector $\C \in \sR^{1\times \dimh}$.
For deep learning applications, discrete $\dA$ and $\dB$ replace continuous $\A$ and $\B$ through discretization methods, such as zero-order hold. 
This yields updated hidden state and output equations as shown in \eqref{eq:rnn_discrete_ssm}.
While \eqref{eq:rnn_discrete_ssm} is structured as linear RNN, it can be reformulated as Convolutional Neural Network (CNN) as \eqref{eq:cnn_discrete_ssm}, enabling efficient and parallelizable training.
SSMs address the quadratic time complexity problem w.r.t sequence length, as the output for each new token depends solely on the hidden states and the current token, in contrast to standard attention mechanisms that attend to all previous tokens.
Consequently, SSMs operate as a causal mechanism.
\begin{align}
\small
    \h'(t) &= \A \vh(t) +  \B x(t), \quad
    y(t) = \C\h(t) \label{eq:rnn_continuous_ssm} \\
    \h_t &= \dA \h_{t-1} + \dB x_{t}, \quad
    y_t = \C\h_t \label{eq:rnn_discrete_ssm} \\
     \dK &= (\C\dB, \C\dA^{\,}\dB, \ldots, \C {\dA}^{t-1}\dB), \\
    &(y_1, \ldots, y_t) = (x_1, \ldots, x_t) * \dK \label{eq:cnn_discrete_ssm}
\end{align}
\textbf{Linear attention.~~} 
Assume a sequence with length $n \in \sN^+$ and embedding size $d \in \sN^+$.
We first focus on non-causal cases.
For the $i$-th position, let $\vq_i \in \sR^d$, $\vk_i \in \sR^d$, and $\vv_i \in \sR^d$ denote the query, key, and value vectors, respectively, where $i = 1,\dots, n$.
In softmax attention, the similarity between $\vq_i$ and $\vk_j$ for any $i \neq j$ is computed as $\exp{(\vq_i^\top \vk_j)}$.
The attention output at the $i$-th position, denoted as $\va_i\in\sR^{d}$, is obtained by averaging the values across all tokens weighted by their similarities.
This process requires $O(n)$ complexity, as it necessitates computing similarities with all $n$ tokens.
Linear attention reduces this complexity by replacing the similarity computation from $\exp(\vq_i^\top \vk_j)$ with $\phi(\vq_i)^\top \phi(\vk_j)$, where $\phi: \sR^{d} \to \sR^d$ is a feature conversion function.
For linear attention outputs~\eqref{eq:linear_attn} across all positions, we identify two common terms: $\sum_{j=1}^n \fe{\vk_j} \cdot \vv_j$ and $\sum_{j=1}^n \fe{\vk_j}$, which can be computed once.
Consequently, for the linear output at position $i$, we only need to compute $\phi(\vq_i)$ and multiply it with these two statistics, resulting in $O(1)$ complexity, thus significantly reducing computational demands. 
\begin{align}
    \text{(Softmax)}~\va_i &= \frac{\sum_{j=1}^n \exp{\left(\vq_i^\top\vk_j\right)} \cdot \vv_j}{\sum_{j=1}^n \exp{\left(\vq_i^\top\vk_j\right)}} \label{eq:linear_attn}\\
    \text{(Linear)}~\va_i &= 
             \frac{\fe{\vq_i}^\top \sum_{j=1}^n \fe{\vk_j} \cdot \vv_j}
         {\fe{\vq_i}^\top \sum_{j=1}^n \fe{\vk_j}}
\end{align}
For causal cases, for position $i$, we replace $\sum_{j=1}^n$ with $\sum_{j=1}^i$, as each token attends only to previous tokens.
The statistics then become $\sum_{j=1}^{i-1} \fe{\vk_j} \cdot \vv_j$ and $\sum_{j=1}^{i-1} \fe{\vk_j}$, which can be viewed as hidden states in RNNs.
Thus, causal linear attention can be conceptualized as a linear RNN, which is also a variant of SSM.

\section{Architectural Exploration for Scalable Tabular Learning}\label{sec:exploration}

\begin{figure*}[t]
    \centering
    \begin{subfigure}[b]{0.32\textwidth}
        \includegraphics[width=\linewidth]{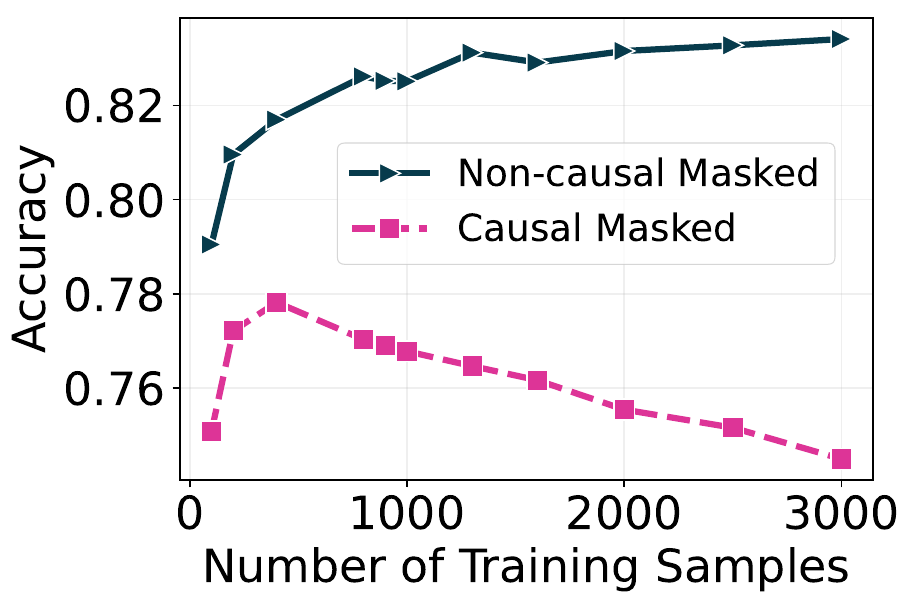}
    \caption{
        \textbf{Effect of causal masking}. When more samples are provided, the non-causal model shows better sample utilization and accuracy, while the causal model's performance plateaus early and declines.
    }
    \label{fig:causal}
    \end{subfigure}\hfill
    \begin{subfigure}[b]{0.32\textwidth}
        \includegraphics[width=0.9\linewidth]{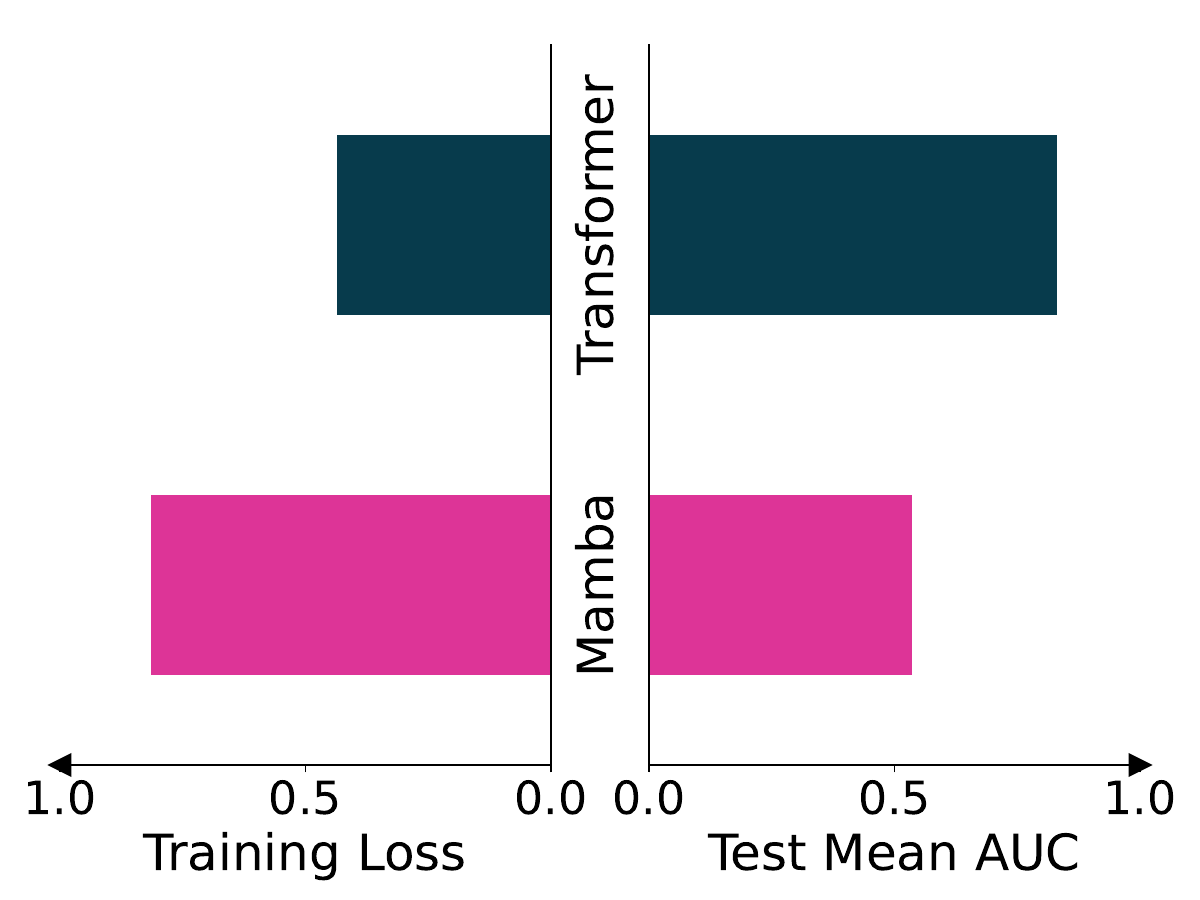}
    \caption{
       \textbf{ICL performance comparison between Mamba and Transformer models}. Results show that Transformer-based models achieve lower training loss and higher AUC across 150 test datasets.
    }
    \label{fig:mamba}
    \end{subfigure}\hfill
    \begin{subfigure}[b]{0.32\textwidth}
        \includegraphics[width=\linewidth]{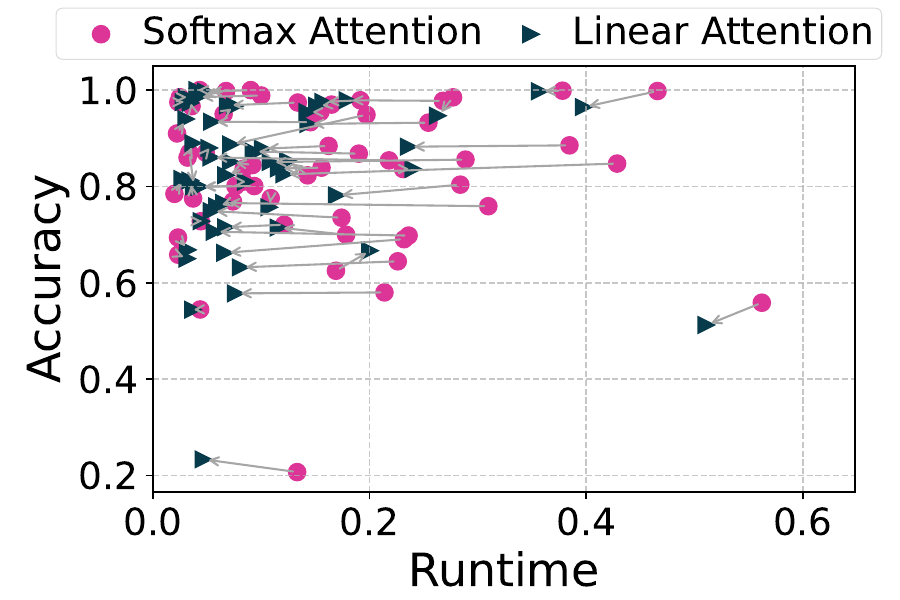}
    \caption{
        \textbf{Accuracy and runtime comparison of softmax and linear attention}. 
        Replacing softmax with linear attention preserves comparable accuracy while significantly reducing runtime.
    }
    \label{fig:attention_type}
    \end{subfigure}
    \vspace{-.1in}
    \caption{\textbf{Impact of model architecture on tabular classification performance.}
    Please refer to Fig.~\ref{fig:attn_explore_detailed} for a detailed breakdown.
    }
    \label{fig:attn_explore}
\end{figure*}

We analyze State-Space Models and linear attention as attention architecture alternatives to enhance the scalability of \tabpfn{}, focusing on tabular classification with ICL. 

\subsection{Causal Model vs. Non-Causal Model}
Ideally, the order of training samples (i.e., in-context demonstrations) in the prompt should not affect the final prediction. 
However, SSMs are inherently causal, computing outputs based on new inputs and hidden states derived from previous inputs. 
This suggests a potential drawback in this context. 
To validate our hypothesis regarding the suboptimal performance of causal models in ICL, we conduct two experiments: 
 (i) comparing the performance of \tabpfn{} with a modified version of the same model that uses causal attention, and 
 (ii) evaluating \tabpfn{} against both its original version and a model incorporating Mamba (specifically Mamba-II), a leading SSM-based architecture.

\paragraph{Causal Attention vs. Non-Causal Attention.}
Our first experiment compares the ICL capabilities of non-causal and causal attention mechanisms using the same experimental setup as \tabpfn{}, shown in Fig.~\ref{fig:causal}.
We replicate \tabpfn{}'s methodology for generating synthetic datasets from priors, training a modified version employing causal attention instead.
For the inference, we generate 20 synthetic datasets that maintain a consistent 1000 test samples with varying numbers of training samples. 
We average the classification accuracy across 20 simulations.

We observe that non-causal attention generally outperforms causal attention. 
As more training samples are given, the accuracy of the non-causal model continues to improve. 
In contrast, the causal attention model shows accuracy improvements only within a very small range of training samples, after which performance begins to decline with additional samples. 
These findings indicate that \tabpfn{} with non-causal attention functions as an effective ICL model, adeptly leveraging context from a large number of samples. Conversely, the same model equipped with causal attention fails to capitalize on the additional data, highlighting the superiority of the non-causal approach in this tabular learning scenario.
Our observation is supported by empirical studies~\citep{ding2024causallm,gong2023improving}, which show that causal attention is suboptimal for ICL. 
Moreover, most theoretical analyses of ICL assume non-causal attention~\citep{ahn2023transformers,bai2023transformers}.

\paragraph{Mamba vs. Transformer.}
We further investigate whether Mamba, the most popular SSM-based model, is suitable for ICL. 
We replicate \tabpfn{}'s training methodology, substituting the transformer layer with an Mamba layer. 
To evaluate performance, we test the modified model on the same 150 validation datasets used in the \tabpfn{} study (See Section F.3 for details). 
Fig.~\ref{fig:mamba} visualizes the training loss and test mean AUC for both methods.
The model with Mamba exhibits significantly higher training loss than the original \tabpfn{}, with substantially lower test mean AUC.
This experiment with a popular SSM model further demonstrates that SSMs underperform non-causal models in our specified tasks.

\subsection{Softmax Attention vs. Linear Attention }

To address the quadratic complexity of standard attention mechanisms, linear attention has emerged as a popular alternative~\citep{katharopoulos2020transformers}. 
To investigate its impact on ICL in tabular classification, we replaced \tabpfn{}'s attention mechanism with linear attention and trained a model following the same strategy as \tabpfn{}. 
We then evaluated both \tabpfn{} and this linear attention model on 57 real datasets (used in Table 2 of \citet{mcelfresh2023tabzilla}, where \tabpfn{} achieved top performance among 19 methods for tabular classification).
Fig.~\ref{fig:attention_type} visualizes the test accuracy and runtime. %
Our results demonstrate that linear attention does not decrease performance and significantly improves speed, making it a suitable method for scaling \tabpfn{} to larger datasets.
Finally, in Section~\ref{app:attention}, we demonstrate that linear attention significantly outperforms sliding window attention~\citep{beltagy2020longformer} in our setting.

\section{\ours{}: Scaling TabPFN for Large Datasets}\label{sec:ours}
Based on the empirical findings presented in Section~\ref{sec:exploration}, we identify non-causal linear attention as the optimal candidate to replace standard softmax attention in \tabpfn{}. 
This section proceeds in two parts: first, we conduct a thorough analysis of the linear attention mechanism to ensure its efficient implementation.; subsequently, we leverage this efficient implementation to train our proposed model, \ours{}. 
Our approach aims to enhance the scalability and performance of tabular learning while maintaining computational efficiency.

\paragraph{Computation Analysis.}
\citet{dao2022flashattention} demonstrates that significant wall-clock speedup for softmax attention can be achieved by optimizing the number of memory reads/writes between GPU high bandwidth memory (HBM) and GPU on-chip SRAM.
Based on this criterion, \citet{yang2024gla} proposed FlashLinearAttention for speeding up \textit{causal} linear attention.
This raises a natural question: can we further improve the speed of non-causal linear attention (we omit non-causal when it does not cause further confusion) by reducing the number of memory reads/writes? 
Our results in Theorem~\ref{thm:computation} analyze the \#HBM access and HBM memory usage of FlashLinearAttention and linear attention, concluding that further optimization is not necessary. 
In Section~\ref{app:complexity}, we first propose an HBM-efficient linear attention, and then show that the PyTorch implementation only incurs a marginal increase in terms of \#HBM access and HBM memory usage, with FLOPS remaining unchanged.
We provide more details, including the analysis of different attention mechanisms and actual memory usage and runtime visualization of these mechanisms in Section~\ref{app:complexity}.
The resulting theorem below demonstrates that the straightforward PyTorch implementation of linear attention achieves linear HBM access, matching the performance of FlashLinearAttention after optimization.
Consequently, our model adopts the straightforward implementation of linear attention.

\begin{theorem}[High Bandwidth Memory Efficiency of Linear Attention]\label{thm:computation}
     Let $\mQ, \mK, \mV \in \sR^{N \times D}$ represent the query, key, and value matrices for a single attention head, where $N$ is the sequence length and $D$ is the embedding size. 
     With any element-wise kernel feature mapping (e.g., $\operatorname{elu}(\cdot) + 1$),
    both causal FlashLinearAttention (Alg.~\ref{alg:fla}) and non-causal linear attention (Listing~\ref{listing:linear_attn}) require $O(ND)$ HBM accesses, $O(ND)$ HBM memory, and $O(ND^2)$ FLOPS to compute the attention output. 
\end{theorem}

\begin{algorithm}
\SetKwInOut{Input}{Input}
\Input{A dataset $\gD$ with $n$ instances of $d$ features }

// Large dataset with few features\;
\uIf{$ n \geq 3$K and $d \leq 100 $}{
    \Return{\ours{}\textnormal{-L100}$(\gD)$}\;
}
// High-dimensional datasets\;
\uElseIf{$d > 100 $ or ($d/n \geq 0.2$ and $n \geq 3$K) }{
    \uIf{$d \leq 1000$}{\Return{\ours{}\textnormal{-H1K}$(\gD)$}\;}
    \Else{Apply random projection to $\gD$ to reduce the number of features to 1000, yielding $\gD'$\;\Return{\ours{}-\textnormal{H1K}$(\gD')$}\;}
}
// Small datasets\;
\Else{
    \Return{\ours{}\textnormal{-S100}($\gD$)}\;
}

\caption{Conditional Model Selection}
\label{alg:tabflex}
\end{algorithm}

\begin{figure}[h]
\centering
    \includegraphics[width=0.85\linewidth]{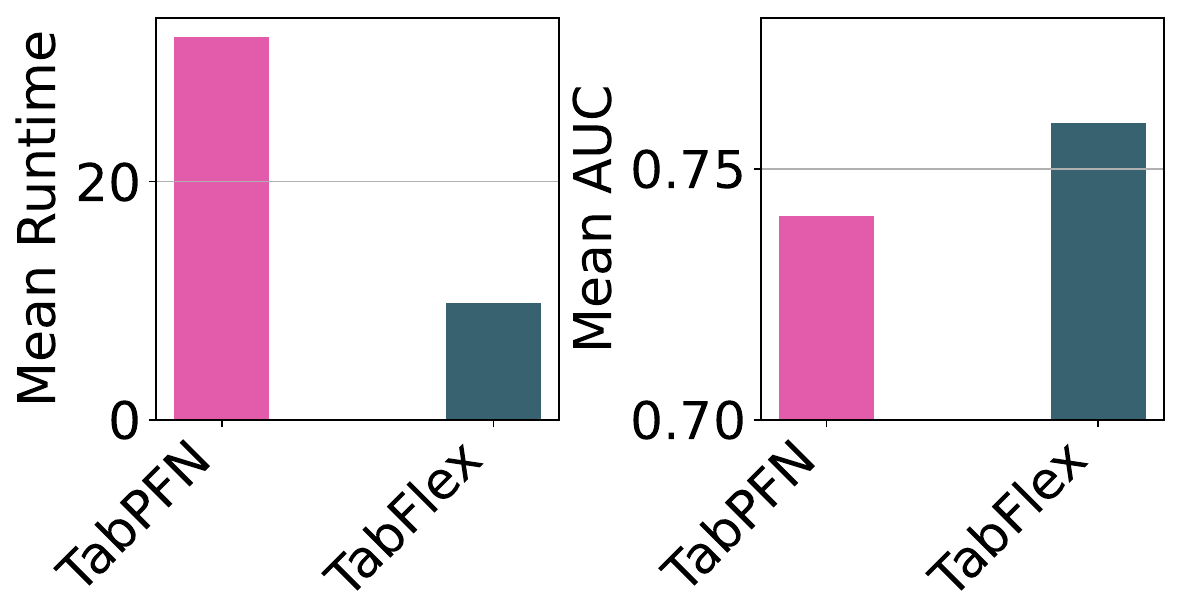}
    \vspace{-0.1in}
    \caption{Runtime and AUC comparison of \tabpfn{} and \ours{} on validation datasets.}
    \label{fig:val}
\end{figure}

\paragraph{\ours{}.}
While \tabpfn{} excels on small, simple datasets with fewer than 100 features and 10 classes, it struggles with more complex tasks, e.g., high-dimensional datasets or those with numerous classes. 
We aim to extend the use cases by training a model that maintains comparable speed to \tabpfn{} while offering reasonable performance across a broader spectrum of datasets. 
Since models trained on large, high-dimensional datasets often struggle in small regions due to optimization issues, we introduce three specialized models to address this limitation.
\begin{itemize}[leftmargin=*]
\item \textbf{\ours{}-S100} is trained on prompts with 1152 length (same as \tabpfn{}), 100 features, and 10 classes. 
This is optimized for low-dimensional datasets. `S' denotes standard configuration, `100' indicates feature capacity.
\item \textbf{\ours{}-L100} utilizes prompts of 50K length, 100 features, and 10 classes. This is designed for large, low-dimensional datasets. `L' signifies a larger sample size, and `100' represents feature count.
\item \textbf{\ours{}-H1K} employs prompts of 50K length, 1K features, and 100 classes. 
This is suited for large, high-dimensional datasets. `H' indicates high-dimensional capabilities, and `1K' denotes 1K features.
\end{itemize}
Additional training details, including training loss, hyperparameters, and other relevant information, are provided in Section~\ref{app:train}.
Our code is available at \url{https://github.com/microsoft/ticl}.

We apply the conditional model selection strategy, shown in Alg.~\ref{alg:tabflex}, to select the model based on the target dataset's size and dimensionality, ensuring optimal performance across diverse data characteristics.
The decision thresholds align with the training regimes of the models. 
\ours{}-S100, sharing TabPFN’s training setup but with an updated architecture, is deployed similarly $(n \leq 3K, d \leq 100)$. 
\ours{}-L100, trained on low-dimensional $(d \leq 100)$ but larger datasets, is used for longer sequences $(n \geq 3K, d \leq 100)$. 
\ours{}-H1K, trained on high-dimensional data, is assigned to handle those cases accordingly.
We note that performance is not highly sensitive to the chosen decision boundaries, supported by our results in Section~\ref{app:sensitivity}.

In Fig.~\ref{fig:val}, we visualize the mean runtime and mean AUC comparison of \tabpfn{} and \ours{} on the validation datasets, comprising 40 datasets with varying sample sizes (up to 100K), dimensions (up to 3K), and number of classes (up to 100). 
Detailed information about these datasets is provided in Section~\ref{app:val}. Our analysis reveals that \ours{} not only exhibits superior performance but also demonstrates faster execution times compared to \tabpfn{}.

\section{Performance Evaluation of \ours{}}\label{sec:exp}
We evaluate \ours{}'s performance and speed across 115 OpenML tabular datasets~\citep{openml}. 
See Section~\ref{App:experiment_setup} for the complete list of baselines, and a detailed description of the models' implementation.

\subsection{Experimental Setup}
Unless otherwise stated, we follow the identical experiment setup of \citet{mcelfresh2023tabzilla} for baseline benchmarking. 

\vspace{-.1in}
\paragraph{Datasets.} 
For classification tasks, we consider datasets of two difficulty levels.
For simpler tasks, we use two collections of datasets—98 and 57 in total—originally reported in Tables 1 and 2 of \citet{mcelfresh2023tabzilla}, which are characterized by smaller sample sizes and lower feature dimensions.
For more challenging tasks, we evaluate the methods on the TabZilla hard benchmark, which includes 36 difficult datasets—11 of which are high-dimensional (100–2000 features) and large-scale ($\geq$ 50K instances).
Detailed dataset information, including names and characteristics, is provided in Section~\ref{app:tabzilla}, with additional datasets and results available in Section~\ref{app:large}.

\paragraph{Baselines.}
We evaluate our approach against a comprehensive set of baselines, as considered by \citet{mcelfresh2023tabzilla}, including 
(i) four classical methods, e.g., Random Forest \citep{liaw2002classification}, 
(ii) three Gradient Boosted Decision Trees (GBDT) methods, e.g., XGBoost \citep{chen2016xgboost}, 
(iii) ten Non-Transformer Neural Network (Non-TF NN) methods, e.g.,  SAINT \citep{somepalli2021saint},
(iv) four Transformer (TF) methods, e.g., \tabpfn{}, with 
two recent methods designed for scaling tabular classification: TuneTables~\citep{feuer2024tunetables}, a TF method, and HyperFast~\citep{bonet2024hyperfast}, a Non-TF NN method.

\begin{table*}[t]
    \centering
    \resizebox{0.9\textwidth}{!}{
\begin{tabular}{l|l|cc|cc|rr}
\toprule
 \multirow{2}{*}{\textbf{Algorithm}} & \multirow{2}{*}{\textbf{Class}} & \multicolumn{2}{c|}{\textbf{Mean AUC}} & \multicolumn{2}{c|}{\textbf{Std. AUC}} & \multicolumn{2}{c}{\textbf{Time / 1000 inst.}} \\
 \cmidrule{3-8}
 & & median & mean & mean & median & median & mean \\
\midrule
TabPFN~\citep{hollmann2023tabpfn} & \tf{} & 0.97 & 0.90 & 0.21 & 0.15 & 0.82 & 1.04 \\
\rowcolor[gray]{0.9} \textbf{\ours{} (Ours)} & \tf{} & 0.96 & 0.89 & 0.22 & 0.16 & 0.29 & 0.48 \\
CatBoost~\citep{prokhorenkova2018catboost} & GBDT & 0.95 & 0.89 & 0.23 & 0.16 & 2.59 & 19.51 \\
ResNet~\citep{resnet} & \nontf{} & 0.93 & 0.84 & 0.24 & 0.16 & 13.90 & 23.40 \\
SAINT~\citep{somepalli2021saint} & \tf{} & 0.93 & 0.84 & 0.24 & 0.20 & 173.63 & 195.16 \\
RandomForest~\citep{liaw2002classification} & Classical & 0.92 & 0.86 & 0.24 & 0.17 & 0.45 & 0.61 \\
XGBoost~\citep{chen2016xgboost} & GBDT & 0.91 & 0.86 & 0.24 & 0.18 & 0.49 & 0.95 \\
HyperFast~\citep{bonet2024hyperfast} & \nontf{} & 0.91 & 0.83 & 0.22 & 0.17 & 64.38 & 136.74 \\
DANet~\citep{chen2022danets} & \nontf{} & 0.89 & 0.80 & 0.25 & 0.19 & 67.70 & 78.21 \\
SVM~\citep{cortes1995support} & Classical & 0.87 & 0.75 & 0.28 & 0.22 & 0.71 & 87.84 \\
NODE~\citep{popov2019neural} & \nontf{} & 0.86 & 0.80 & 0.24 & 0.18 & 157.18 & 194.07 \\
DeepFM~\citep{guo2017deepfm} & \nontf{} & 0.86 & 0.79 & 0.28 & 0.27 & 5.48 & 5.95 \\
FTTransformer~\citep{gorishniy2021revisiting} & \tf{} & 0.84 & 0.78 & 0.25 & 0.21 & 25.40 & 33.34 \\
LightGBM~\citep{ke2017lightgbm} & GBDT & 0.83 & 0.76 & 0.28 & 0.21 & 0.25 & 0.67 \\
MLP-rtdl~\citep{gorishniy2021revisiting} & \nontf{} & 0.83 & 0.74 & 0.26 & 0.20 & 12.65 & 22.97 \\
LinearModel~\citep{cox1958regression} & Classical & 0.81 & 0.71 & 0.27 & 0.21 & 0.05 & 0.06 \\
TuneTables~\citep{feuer2024tunetables} & \tf{} & 0.80 & 0.72 & 0.32 & 0.24 & 53.48 & 113.49 \\
STG~\citep{yamada2020feature} & \nontf{} & 0.79 & 0.67 & 0.29 & 0.23 & 18.46 & 21.26 \\
TabTransformer~\citep{huang2020tabtransformer} & \tf{} & 0.79 & 0.64 & 0.24 & 0.16 & 19.04 & 32.84 \\
MLP~\citep{mlp} & \nontf{} & 0.72 & 0.65 & 0.29 & 0.25 & 17.83 & 27.67 \\
DecisionTree~\citep{quinlan1986induction} & Classical & 0.63 & 0.55 & 0.35 & 0.31 & 0.01 & 0.02 \\
KNN~\citep{cover1967nearest} & Classical & 0.62 & 0.56 & 0.30 & 0.25 & 0.03 & 0.03 \\
TabNet~\citep{arik2021tabnet} & \tf{} & 0.56 & 0.50 & 0.42 & 0.40 & 34.66 & 42.09 \\
VIME~\citep{yoon2020vime} & \nontf{} & 0.49 & 0.48 & 0.37 & 0.27 & 18.43 & 20.11 \\
NAM~\citep{agarwal2021neural} & \nontf{} & 0.33 & 0.38 & 0.38 & 0.31 & 147.30 & 341.58 \\
\bottomrule
\end{tabular}}
    \caption{
    \textbf{Performance of algorithms across 57 datasets of size less than or equal to 1250 (used in Table 2 of \citet{mcelfresh2023tabzilla}). }
    The reported AUC values are normalized.
    The ``Time/1000 inst.'' column represents the combined training and test time for all datasets, divided by the total number of samples.
    Notably, \ours{} achieves top-2 performance, with significantly faster runtimes compared to baselines of similar performance, and a 2$\times$ speedup relative to \tabpfn{}.
}
    \label{tab:tabzilla2}
\end{table*}

\subsection{Evaluation on Simple Datasets}\label{sec:tabzilla_simple}

We evaluate models on two sets of data: 98 simple datasets from Table 1 and 57 small datasets from Table 2 of \citet{mcelfresh2023tabzilla}.
The results are reported in Table~\ref{tab:tabzilla1} (Section~\ref{App:sec_detailed_result}) and Table~\ref{tab:tabzilla2}, respectively.
For each dataset, we consider ten different train/test splits, computing the score mean and standard deviation, as well as the total runtime per 1000 instances.
We then calculate the median and mean of these values across the entire set of datasets: 98 simple datasets for Table~\ref{tab:tabzilla1} and 57 small datasets for Table~\ref{tab:tabzilla2}.
Algorithms are ranked based on AUC and time.
Our results demonstrate that \ours{} achieves nearly identical performance to \tabpfn{} on small, simple datasets while offering more than a 2$\times$ speedup.
Compared to faster methods, such as Decision Tree and Linear Model in Table~\ref{tab:tabzilla1}, and Decision Tree, Linear Model, LightGBM, and KNN in Table~\ref{tab:tabzilla2}, their performance is significantly inferior to \ours{}. 

\subsection{Evaluation on Hard Datasets}\label{sec:tabzilla_hard}
\begin{figure}
    \centering
    \includegraphics[width=0.95\linewidth]{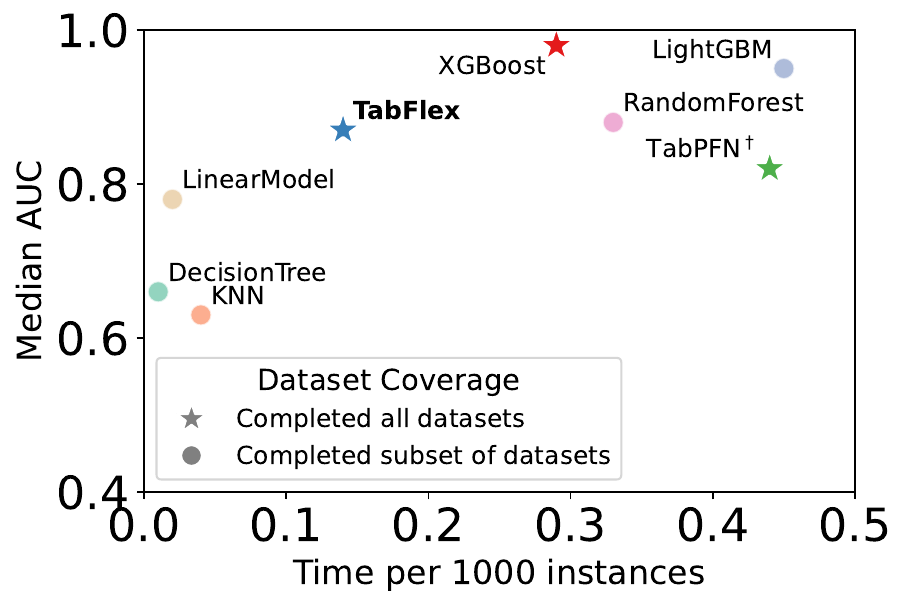}
    \captionsetup{skip=3pt}
    \vspace{-.05in}
    \caption{
    \textbf{Visualization of tested methods with processing times under 0.5 seconds per 1000 instances on the TabZilla hard benchmark.}
    We report the median AUC across the completed datasets, as several methods completed only a subset of the datasets.
    Compared to other methods (XGBoost and \tabpfn{}) that successfully ran on all datasets, \ours{} achieves a 2$\times$ speedup while maintaining relatively good performance.
    }
    \label{fig:tabzilla4}
    \vspace{-.1in}
\end{figure}
In this experiment, we compare \ours{} to baselines on the TabZilla hard benchmark~\citep{mcelfresh2023tabzilla}, which includes 36 datasets.
However, due to the challenging nature of the datasets in the TabZilla hard benchmark, many baselines fail to execute successfully.
In Fig.~\ref{fig:tabzilla4}, we visualize the Median AUC and the runtime per 1000 instances across the 36 datasets, with methods that successfully executed on all datasets marked as stars, and methods that failed to execute on some datasets marked as circles.
This figure focuses on efficient methods, excluding those slower than 0.5 seconds per 1000 instances.
We observe that only \ours{}, \tabpfn{}, and XGBoost successfully run on all datasets.
Notably, \ours{} is faster and achieves better performance than \tabpfn{}, and is faster than XGBoost while sacrificing only a small margin of performance.

\begin{table*}[t]
    \centering
\resizebox{0.95\textwidth}{!}{
\begin{tabular}{l|rrr|ccc|rrr}
\toprule
\multirow{2}{*}{\textbf{Dataset}} & \multirow{2}{*}{\textbf{\#Classes}} & \multirow{2}{*}{\textbf{\#Features}} & \multirow{2}{*}{\textbf{\#Instances}} & \multicolumn{3}{c|}{\textbf{AUC}} & \multicolumn{3}{c}{\textbf{Time (seconds)}}  \\ \cmidrule{5-10}
&&&& 5th Best & \tabpfn{}  & \ours{} & 5th Best & \tabpfn{} & \ours{} \\ 
\midrule
SpeedDating & 2 & 120 & 8378 & 0.86 & 0.55 & 0.85 & 1.58 & 1.58 & 1.89 \\
higgs & 2 & 28 & 98050 & 0.79 & 0.72 & 0.76 & 3.46 & 2.82 & 4.92 \\
cnae-9 & 9 & 856 & 1080 & 1.00 & 0.48 & 0.96 & 0.51 & 0.51 & 3.80 \\
albert & 2 & 78 & 425240 & 0.71 & 0.69 & 0.70 & 33.98 & 9.39 & 13.46 \\
audiology & 24 & 69 & 226 & 0.92 & 0.82 & 0.81 & 0.13 & 0.23 & 0.26 \\
jasmine & 2 & 144 & 2984 & 0.86 & 0.70 & 0.86 & 0.68 & 1.27 & 0.99 \\
nomao & 2 & 118 & 34465 & 0.99 & 0.76 & 0.99 & 4.03 & 1.82 & 5.34 \\
Bioresponse & 2 & 1776 & 3751 & 0.85 & 0.50 & 0.75 & 2.49 & 1.29 & 12.38 \\
MiniBooNE & 2 & 50 & 130064 & 0.98 & 0.98 & 0.97 & 10.80 & 3.19 & 7.22 \\
airlines & 2 & 7 & 539383 & 0.70 & 0.63 & 0.64 & 6.53 & 9.73 & 4.20 \\
\rowcolor[gray]{0.9} poker-hand & 10 & 10 & 1025009 & 0.54 & 0.72 & 0.84 & 504.52 & 15.36 & 4.88 \\
\bottomrule
\end{tabular}
}
 \caption{
 \textbf{Performance comparison of \ours{}, \tabpfn{}, and other baselines on large, high-dimensional datasets from the TabZilla hard benchmark~\citep{mcelfresh2023tabzilla}. }
Baseline results are summarized by the 5th highest AUC and the 5th lowest runtime for each dataset. 
\ours{} significantly outperforms \tabpfn{} on these datasets, achieving comparable performance to other baselines while maintaining exceptional speed.
}
    \label{tab:tabzilla4}
\end{table*}

Next, we focus on 11 high-dimensional and large datasets within the TabZilla hard benchmark.
Since most baselines do not obtain complete results for all datasets, instead of comparing \ours{} to a specific baseline, we report the 5th-best AUC and 5th-best runtime, using these values to summarize the general performance distribution of the baselines.
The results are presented in Table~\ref{tab:tabzilla4}.
We observe that, for these datasets, \ours{} substantially outperforms \tabpfn{}.
While \tabpfn{} follows \citet{mcelfresh2023tabzilla}'s strategy of using only 3000 training samples, \ours{} utilizes all available training data, achieving superior performance with comparable or slightly higher processing times.
\ours{} exhibits competitive performance among baselines while maintaining high efficiency.
Notably, on large datasets with more than 50K instances, \ours{} is significantly faster than the baselines.
For instance, on the largest dataset, \textit{poker-hand}, containing over one million samples, \ours{} significantly outperforms other baselines, classifying all samples in just 4.88 seconds, while the fifth fastest method requires more than 500 seconds.

\begin{table}[t]
\centering
\resizebox{\linewidth}{!}{
\begin{tabular}{lccc}
\toprule
\textbf{Dataset} & \textbf{\ours{}} & \textbf{Linear Regression} & \textbf{XGBoost} \\
\midrule
cpu\_act & 0.9622 & 0.7661 & 0.9872 \\
pol & 0.7770 & 0.4471 & 0.9876 \\
elevators & 0.7386 & 0.8336 & 0.8984 \\
wine\_quality & 0.1966 & 0.2842 & 0.4398 \\
Ailerons & 0.7284 & 0.8137 & 0.8272 \\
houses & 0.6803 & 0.6496 & 0.8469 \\
house\_16H & 0.2519 & 0.1708 & 0.5276 \\
diamonds & 0.9085 & 0.9213 & 0.9477 \\
Brazilian\_houses & 0.8943 & 0.3459 & 0.9828 \\
Bike\_Sharing\_Demand & 0.3796 & 0.3291 & 0.6995 \\
nyc-taxi-green-dec-2016 & 0.1547 & 0.3109 & 0.5732 \\
house\_sales & 0.6656 & 0.7375 & 0.8732 \\
sulfur & 0.4026 & 0.3068 & 0.7497 \\
medical\_charges & 0.8173 & 0.8118 & 0.9790 \\
MiamiHousing2016 & 0.8112 & 0.7302 & 0.9306 \\
superconduct & 0.6867 & 0.7169 & 0.9086 \\
yprop\_4\_1 & 0.0000 & 0.0449 & 0.0000 \\
abalone & 0.3689 & 0.4622 & 0.5125 \\
\bottomrule
\end{tabular}}
\caption{
\textbf{Comparison of performance across regression tasks for \ours{}, Linear Regression, and XGBoost.}
Regression datasets are from \citet{grinsztajn2022tree}.
To extend \ours{} to regression, we discretize the target variable into 10 and 100 uniform bins and use the better-performing setting, converting the task to classification.
Despite being designed for classification, 
\ours{} delivers reasonable performance on regression tasks.}
\label{tab:regression_results}
\end{table}

\subsection{Extension to Regression Tasks}
So far, we have presented our evaluation on classification tasks.
For the extension to regression, a simple workaround is to convert the task into classification by discretizing the target range into bins.
Shown in Table~\ref{tab:regression_results}, we apply this technique for regression data with numerical features from \citet{grinsztajn2022tree}, where targets are discretized into 10 and 100 uniform bins.
For baselines, we use linear regression and the XGBoost Regressor (100 estimators, max depth 6), both with default parameters from the \texttt{scikit-learn} package.
Although regression is not the primary goal of \ours{}, it demonstrates reasonable performance.

Furthermore, we compare \ours{} with TuneTables~\citep{feuer2024tunetables} using their setup in Sec.~\ref{app:tunetables} and extend \ours{} to image datasets, shown in Section~\ref{app:image}.

\section{Ablation Studies}\label{sec:ablation}
We conduct ablation studies, including a fine-grained comparison with XGBoost and the integration of other data-efficient techniques. Performance and runtime trends with respect to training sample sizes, along with detailed experimental setups, are provided in Section~\ref{app:ablation}.

\subsection{Fine-Grained Comparison with XGBoost}
In Fig.~\ref{fig:tabzilla4}, we observe a larger performance gap between \ours{} and XGBoost compared to the simpler datasets shown in Table~\ref{tab:tabzilla2}.
To better understand this discrepancy, we conduct a more fine-grained comparison between \ours{} and XGBoost using synthetic datasets.
XGBoost is configured with a tree depth of 3 and 20 estimators to balance speed and accuracy.
We evaluate the accuracy-runtime tradeoff across varying feature dimensions (${100, 200, 400, 600, 800, 1000}$) and training sample sizes (${1000, 2000, \ldots, 12000}$), averaging results over 20 synthetic datasets with diverse distributions.
\ours{} consistently outperforms XGBoost in both accuracy and inference time when the feature dimensionality is below 800.
As the number of features increases, the performance gap narrows, and XGBoost eventually surpasses \ours{} at 800 features.
Nevertheless, \ours{} achieves a stronger overall tradeoff across most settings.

\begin{figure}
\centering
    \begin{subfigure}[b]{0.33\linewidth}
        \includegraphics[width=\linewidth]{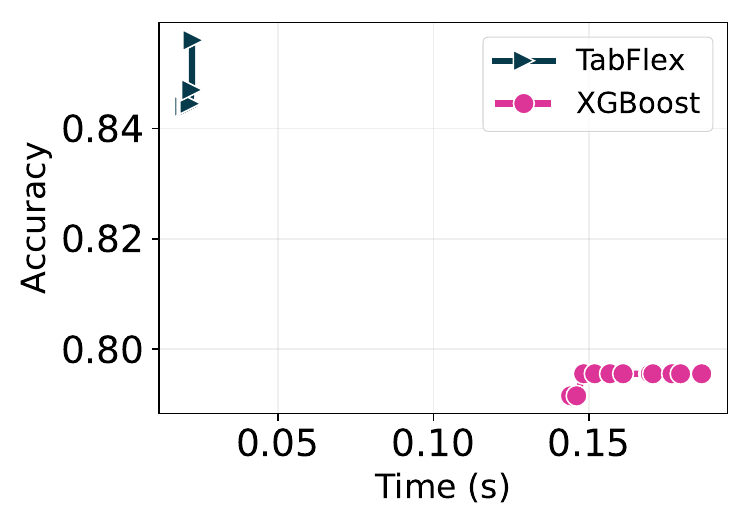}
        \caption{100 features}
    \end{subfigure}\hfill
    \begin{subfigure}[b]{0.33\linewidth}
        \includegraphics[width=\linewidth]{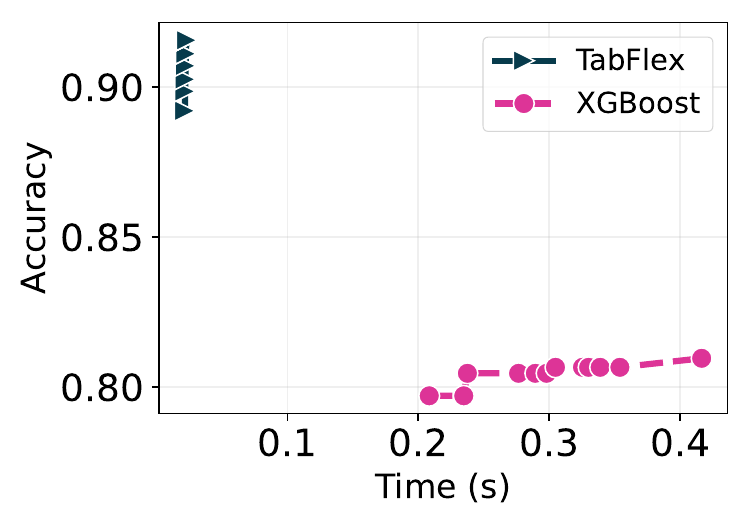}
        \caption{200 features}
    \end{subfigure}\hfill
    \begin{subfigure}[b]{0.33\linewidth}
        \includegraphics[width=\linewidth]{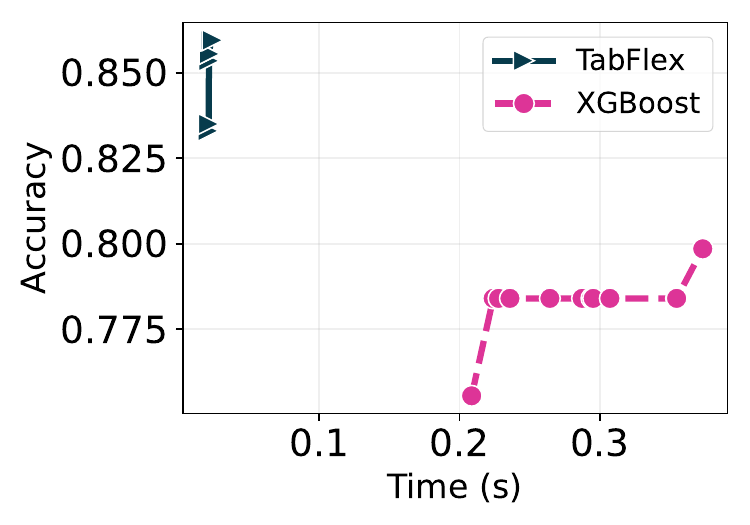}
        \caption{400 features}
    \end{subfigure}\hfill
    \begin{subfigure}[b]{0.33\linewidth}
        \includegraphics[width=\linewidth]{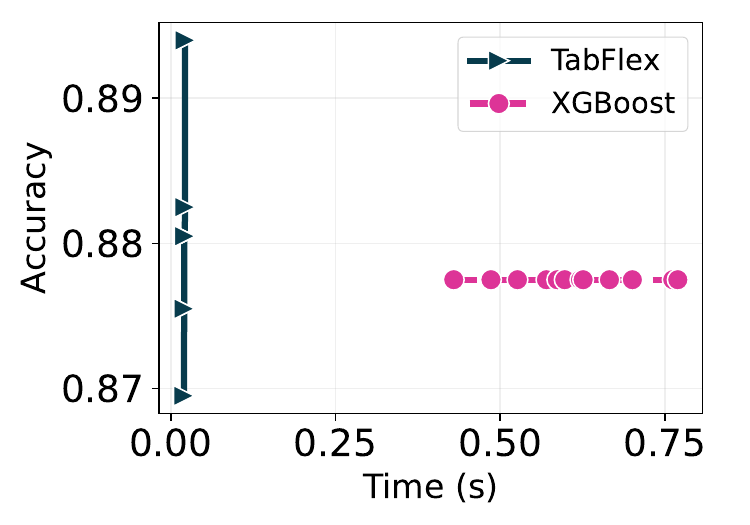}
        \caption{600 features}
    \end{subfigure}\hfill
    \begin{subfigure}[b]{0.33\linewidth}
        \includegraphics[width=\linewidth]{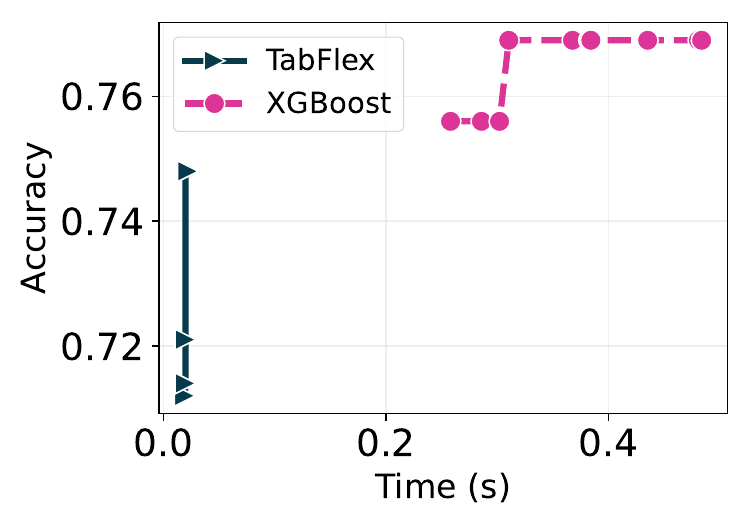}
        \caption{800 features}
    \end{subfigure}\hfill
    \begin{subfigure}[b]{0.33\linewidth}
        \includegraphics[width=\linewidth]{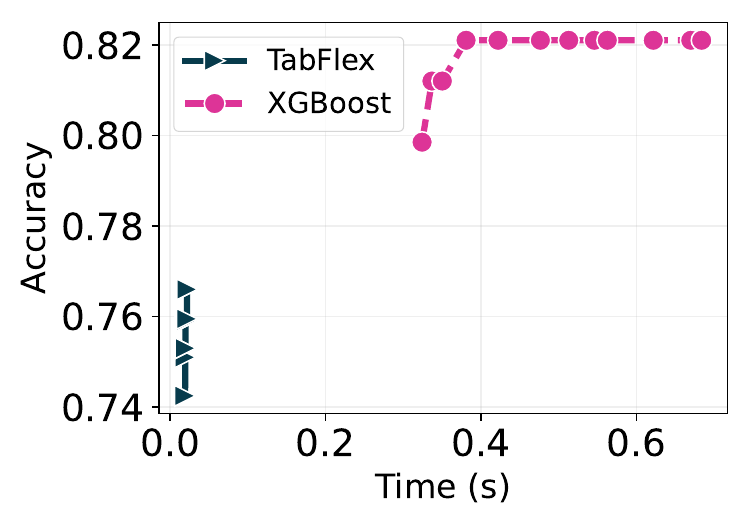}
        \caption{1000 features}
    \end{subfigure}
    \vspace{-0.1in}
    \caption{
        \textbf{Accuracy-runtime tradeoff of \ours{} and XGBoost across feature dimensions (${100, 200, 400, 600, 800, 1000}$) and training sample sizes (${1000, 2000, \ldots, 12000}$).}
        Results are averaged over 20 synthetic datasets with diverse data distributions.
        \ours{} consistently achieves better performance and faster inference than XGBoost when the number of features is below 800.
        As dimensionality increases, the gap diminishes, with XGBoost overtaking \ours{} at 800 features.
    }
    \label{fig:speedup}
\end{figure}

\subsection{Incorporating Data-Efficient Techniques: Dimensionality Reduction and Data Sampling}
\label{sec:exp_ab}

Since \ours{} utilizes the ICL for prediction, reducing the complexity of the data can further improve the inference efficiency.
We combine \ours{} with two data-efficient techniques (feature dimension reduction and training data sampling) and investigate how they affect the balance between efficiency and classification performance. See Section~\ref{App:dataset_ab} for detailed setup and datasets.

First, for dimensionality reduction, we apply three techniques: Principal components analysis (PCA)~\cite{mackiewicz1993principal}, Singular Value Decomposition (SVD), and random linear projection~\cite{vempala2005random}.
Fig.~\ref{fig:ab_dim_reduction} presents the performance and latency of \ours{} where feature dimensions are varied.
Datasets are selected from Table~\ref{tab:dataset_tabzilla1} where feature dimensions are greater than 100.
On average, dimensions can be reduced to 10\% to reduce latency more than 2 times while AUC scores remain the same or even better than the original predictions (100\%).

Similarly, we conduct ablation studies on training data size with different sampling methods (K-centers, K-medoids, uncertainty sampling, and random sampling), shown in Fig.~\ref{fig:ab_sampling}.
The tested datasets are selected from Table~\ref{tab:dataset_tabzilla1} where
the data size is greater than 1000 instances and the feature dimension is lower than 100. 
Fig.~\ref{fig:ab_sampling} illustrates that the original performance can be preserved with only 20\% of training data while the latency can be significantly reduced.

\begin{figure}[t]
    \centering
    \includegraphics[width=0.49\linewidth]{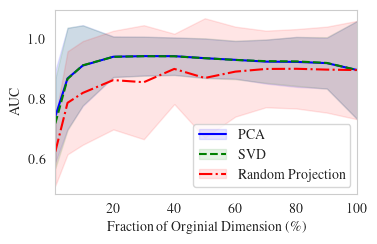}
    \includegraphics[width=0.49\linewidth]{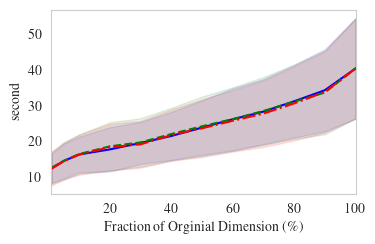}
    \vspace{-0.1in}
    \caption{\textbf{Varying feature dimension with dimensionality reduction methods.} Results are measured on a set of datasets whose number of features is greater than 100. Dimension of features can be reduced up to 90\% to preserve the performance (left) with inference being $2\times$ faster (right).}
    \label{fig:ab_dim_reduction}
\end{figure}

\begin{figure}[t]
    \centering
    \includegraphics[width=0.49\linewidth]{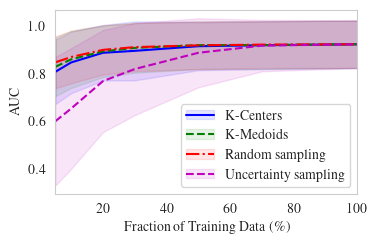}
    \includegraphics[width=0.49\linewidth]{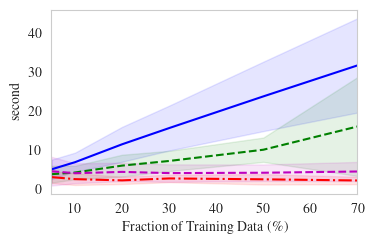}
    \vspace{-0.05in}
    \caption{\textbf{Varying training data with sampling methods.} Results are measured on a set of datasets whose data size is greater than 1000 and whose dimension is lower than 100. For most datasets, only 20\% of training data may be required to preserve the performance (left), while significantly reducing the latency (right).
    }
    \label{fig:ab_sampling}
\end{figure}

\section{Conclusion \& Discussion} 
\paragraph{Conclusion.}
To extend \tabpfn{} for ICL on larger and more challenging tabular classification tasks, in this paper, we conduct a comprehensive exploration of scalable alternatives to attention, ultimately selecting non-causal linear attention.
Through computational analysis for algorithmic optimization of the implementation of linear attention, we develop our model, \ours{}.
We demonstrate that \ours{} achieves comparable performance to \tabpfn{} on small datasets with more than 2$\times$ speedup, while outperforming most other baselines with significantly reduced computational time.
Moreover, \ours{} significantly outperforms \tabpfn{} on larger and more complex datasets, becoming much faster than most other baselines on datasets larger than 100K samples, while maintaining performance on par with state-of-the-art methods.
We posit that \ours{} further elevates the performance ceiling of neural network-based models on tabular classification tasks.

\paragraph{Limitations \& Future Works.}
While our method achieves fast inference and competitive performance on datasets with around two thousand features, scaling to even larger feature spaces remains a compelling research direction. 
In this paper, we also explored adapting \ours{} to image classification tasks on small-scale datasets such as MNIST and CIFAR-10. 
Extending this approach to large-scale image classification could broaden its applicability, especially given its extremely fast inference and ability to generate predictions for all test samples simultaneously.
Further extending \ours{} to other modalities such as audio and video is another promising direction. 
This may require new strategies for synthetic data generation for pretraining, along with systematic analyses of the effects of architectural hyperparameters such as depth and embedding size. These efforts would help optimize the model for high-dimensional inputs and enhance its generalization across domains.
Currently, our focus on tabular tasks is limited to classification.
We attempted a naive extension to regression by discretizing the target range into bins and treating it as a classification problem.
A promising future direction is to extend \ours{} to regression tasks using a more principled approach, such as training on regression-specific synthetic data.
Lastly, TabPFNv2 is a concurrent work that further improves the performance of \tabpfn{}.
Investigating how incorporating linear attention might impact TabPFNv2 is also an interesting question for future research.

\newpage

\section*{Impact Statement}
This research tackles a key limitation in applying \tabpfn{} to tabular classification tasks: scalability.
By introducing \ours{}, which leverages linear attention, it enables efficient processing of tabular datasets with millions of samples and thousands of features.
The primary impact lies in making in-context learning feasible for a broader range of large-scale tabular problems, which is particularly useful in domains such as finance and recommendation systems.
Furthermore, the successful integration of linear attention in \ours{} lays the groundwork for future studies on efficient attention mechanisms and model architectures for tabular and other non-NLP domains, as supported by promising preliminary results in image classification.

\bibliography{references}
\bibliographystyle{icml2025}

\newpage
\appendix
\onecolumn
\newpage
\onecolumn
\appendix
{\LARGE \textbf{Appendix}} \par 

\makeatletter
\def\addcontentsline#1#2#3{%
  \addtocontents{#1}{\protect\contentsline{#2}{#3}{\thepage}{\@currentHref}}%
}
\makeatother
\startcontents[sections]
\printcontents[sections]{ }{1}{}
\newpage

\section{Extended Section~\ref{sec:related_work}: Related Works}\label{app:related_works}
\paragraph{Classical Machine Learning Approaches for Tabular Classification.}
Classical machine learning algorithms have long been the foundation of tabular data classification. 
These methods include k-Nearest Neighbors (KNN)~\citep{cover1967nearest}, Logistic Regression~\citep{cox1958regression}, Decision Trees~\citep{quinlan1986induction}, and Support Vector Machines (SVM)~\citep{cortes1995support}. 
These classical models, while effective, often struggle to handle complex, high-dimensional tabular datasets, motivating the development of more sophisticated approaches. 

\paragraph{Gradient-Boosting Decision Trees for Tabular Classification}
Gradient-boosting decision trees (GBDTs)~\citep{friedman2001greedy} have emerged as a cornerstone in tabular classification, owing to their exceptional ability to capture intricate patterns in structured data. 
GBDTs refine their outputs to minimize errors by iteratively combining predictions from weak learners, resulting in high predictive accuracy. 
XGBoost~\citep{chen2016xgboost} introduced weighted quantile sketching, advanced regularization techniques, and sparsity awareness, achieving state-of-the-art performance. 
LightGBM~\citep{ke2017lightgbm}, a computationally efficient GBDT implementation, employs Gradient-based One-Side Sampling and a leaf-wise tree growth strategy. 
CatBoost~\citep{prokhorenkova2018catboost} leverages symmetric trees and introduces ordered boosting, with a particular emphasis on effectively handling categorical features. 
These advancements have rendered GBDTs not only powerful but also versatile tools in the domain of tabular data, dominating tabular classification in terms of both speed and performance until the advent of \tabpfn{}.

\paragraph{Transformer-based Approaches for Tabular Classification.}
Recent years have witnessed numerous attempts to employ Transformers for tabular classification~\citep{arik2021tabnet,huang2020tabtransformer,gorishniy2021revisiting,dinh2022lift,hollmann2023tabpfn}. 
These methods utilize Transformers in diverse ways to tackle tabular data. 
TabNet~\citep{arik2021tabnet}, one of the pioneering efforts, applies unsupervised pre-training on masked tabular datasets to infer missing features, thereby enhancing the model's understanding of datasets and features. 
It then performs supervised learning on feature selection to obtain the final decision boundary, akin to decision trees. 
\citet{huang2020tabtransformer} introduced TabTransformer, which leverages Transformers to better handle categorical features by concatenating their contextual embeddings with numerical features. 
While TabTransformer processes categorical and continuous features separately, SAINT \citep{somepalli2021saint} projects both feature types into a shared embedding space before passing them through transformer blocks, thereby enhancing overall performance.
FT-Transformer~\citep{gorishniy2021revisiting} introduces a feature tokenizer to convert each example into a sequence of embeddings, enabling Transformers to process tabular datasets and make predictions. 
LIFT~\citep{dinh2022lift} utilizes a pre-trained language model with parameter-efficient fine-tuning, incorporating task descriptions and converting each sample into a complete sentence with feature names in the prediction prompt.
TabR~\citep{gorishniy2024tabr} proposes a retrieval-augmented model with a custom kNN-like component to retrieve and extract signals from the nearest neighbors. 
BiSHop~\citep{xu2024bishop} establishes interconnected directional learning modules to process data column-wise and row-wise for tabular learning.
XTab~\citep{zhu2023xtab} utilizes independent featurizers and federated learning to resolve inconsistent column types and quantities.

\tabpfn{}~\citep{hollmann2023tabpfn} is trained offline on synthetic datasets derived from prior distributions and performs ICL rather than additional parameter tuning for a given dataset, enabling it to solve small tabular classification tasks within seconds.
Prior to our work, TuneTable~\citep{feuer2024tunetables} extended \tabpfn{} to scale to large datasets by performing prefix-tuning for each dataset to achieve better performance.
MixturePFN~\citep{xu2025mixturepfn} improves scalability by routing new test samples to a pool of scalable prompters using Sparse Mixture of In-Context Prompters, while LoCalPFN~\citep{thomas2024retrieval} proposes retrieving a local subset of task-specific data for efficiently fine-tuning on. 
\citet{ma2024context} introduce in-context data distillation to optimize TabPFN’s context and remove the data size constraint.
TuneTable~\citep{feuer2024tunetables} scales \tabpfn{} to large datasets by performing a prefix tuning per dataset.
TabPFNv2~\citep{tabpfnv2} enhances TabPFN’s accuracy in low-data regimes (fewer than 10,000 samples), complementing our focus on speed and scalability.
Our method is also based on \tabpfn{}, extending its scalability to large datasets while maintaining and improving efficiency by simply replacing the softmax attention with linear attention.

\paragraph{Attention Mechanisms and Scalable Alternatives.}
While attention in Transformers~\citep{vaswani2017attention} is central to the strong performance of language models, it encounters scaling challenges for long sequences due to its quadratic computational and memory complexity. To overcome these limitations, several scalable alternatives have been proposed~\citep{gu2024mamba,dao2024transformers,katharopoulos2020transformers,peng2023rwkv,orvieto2023resurrecting,sun2023retentive}, all aiming to achieve subquadratic time complexity.
In contrast, classical RNNs provide the advantage of efficient linear-time inference but suffer from limitations in training efficiency, lacking the parallelization capabilities of Transformer architectures. 
Linear attention~\citep{katharopoulos2020transformers} addresses both concerns by reformulating self-attention as a linear dot-product of kernel feature maps, reducing the computational complexity from quadratic to linear time. 
Additionally, causal linear attention can be interpreted as a form of RNN, as the model makes predictions based on a current token and a ``hidden state,'' which summarizes information from the previous tokens. 
State-space models (SSMs), another popular variant of RNNs, address the drawbacks of classical RNNs by considering linear RNNs and proposing novel algorithms for efficient training~\citep{gu2021combining,gu2022efficiently,gu2024mamba,dao2024transformers,peng2023rwkv,orvieto2023resurrecting,sun2023retentive}. 

\citet{dao2022flashattention} identified that another bottleneck in attention mechanisms' speed stems from the relatively slow access to high-bandwidth memory (HBM) in GPUs.
To address this limitation, FlashAttention~\citep{dao2022flashattention,dao2024flashattention,shah2024flashattention} restructures attention computation to optimize the utilization of high-speed on-chip SRAM while minimizing access to slower HBM, thereby enhancing the efficiency of GPU-based attention operations.
FlashAttention strategically balances computational efficiency against memory bandwidth efficiency.
Although the computational complexity in terms of sequence length remains quadratic, the optimizations introduced by FlashAttention significantly accelerate attention computation in wall-clock time.

\paragraph{Non-Transformer Neural Network-based Approaches for Tabular Classification.}
Non-Transformer neural networks, such as Multi-Layer Perceptrons (MLP)~\citep{mlp}, were explored for tabular classification long before Transformer-based methods, but their performance was limited. In recent years, several novel neural network techniques have been developed for this task, including ResNet~\citep{resnet}, DANet~\citep{chen2022danets}, NODE~\citep{popov2019neural}, DeepFM~\citep{guo2017deepfm}, STG~\citep{yamada2020feature}, VIME~\citep{yoon2020vime}, and NAM~\citep{agarwal2021neural}. DeepFM~\citep{guo2017deepfm} employs a factorization machine-based neural network to learn from categorical data. Drawing inspiration from CatBoost, \citet{popov2019neural} presents a novel neural network architecture designed specifically for tabular data, named Neural Oblivious Decision Ensembles (NODE). While self- and semi-supervised learning have demonstrated effectiveness in the domains of computer vision and natural language processing, \citet{yoon2020vime} proposed Value Imputation and Mask Estimation (VIME), which represents the first attempt to address tabular tasks using a self- and semi-supervised learning framework. \cite{agarwal2021neural} proposed the Neural Additive Model (NAM), an interpretable neural network that maintains strong performance on tabular data. \citet{yamada2020feature} proposed a feature selection method using stochastic gates (STG), which is a neural network-based and effective approach for tabular data. \citet{chen2022danets} designed an abstract layer, a specialized neural component for tabular data, and proposed Deep Abstract Networks (DANets) by stacking these layers.

Some approaches even replace Transformers with SSMs for tabular learning~\citep{ahamed2024mambatab,thielmann2024mambular}. However, these methods require training on a per-dataset basis, leading to high computational costs, and they are generally slower than GBDTs for tabular classification tasks.

\paragraph{Linear Attention for In-Context Learning.} Although linear attention has been reported to underperform in some language modeling tasks~\citep{you2024when,zhang2024the,qin2022devil}, recent theoretical work demonstrates its effectiveness in in-context learning scenarios, where it can emulate gradient descent to achieve learning during inference~\citep{ahn2023transformers}.

\section{Supplement to Section~\ref{sec:exploration}: Architectural Exploration for Scalable Tabular Learning}\label{app:exploration}
\subsection{Comparison with Other Attention Variant}\label{app:attention}
In addition to the broad categories of all linear RNN variant models we studied in this paper, we also consider another mechanism that enjoys linear complexity: sliding window attention~\citep{beltagy2020longformer}. We show that \ours{} achieves significantly better performance.

\begin{table}[h!]
    \centering
    \begin{tabular}{l|ccc|cc}
        \toprule
        \textbf{Method} & \textbf{\#Class} & \textbf{\#Features} & \textbf{\#Instances} & \textbf{Sliding Window} & \textbf{Linear (Ours)} \\
        \midrule
        Poker-Hand & 10 & 10 & 1,025,009 & 0.48 & \textbf{0.84} \\
        Airlines   & 2  & 7  & 539,383   & 0.48 & \textbf{0.64} \\
        Higgs      & 2  & 28 & 98,050    & 0.39 & \textbf{0.76} \\
        \bottomrule
    \end{tabular}
    \caption{Performance comparison of \ours{} with Sliding Window attention.}
\end{table}

\subsection{Supplementary Visualizations}
Fig.~\ref{fig:attn_explore_detailed} provides additional visualizations to complement the analysis in Fig.~\ref{fig:attn_explore}. 
It includes both TabPFN-Causal-Masked and TabPFN-Mamba variants in the settings shown in Fig.~\ref{fig:causal} and Fig.~\ref{fig:mamba}, as the two experiments highlight different aspects of model behavior. 
To improve the interpretability of performance differences between softmax and linear attention across datasets, we also include a density plot illustrating the changes in runtime and accuracy when the softmax operator is replaced with linear attention.
\begin{figure}[t]
    \centering
    \begin{subfigure}[b]{0.49\textwidth}
        \includegraphics[width=\linewidth]{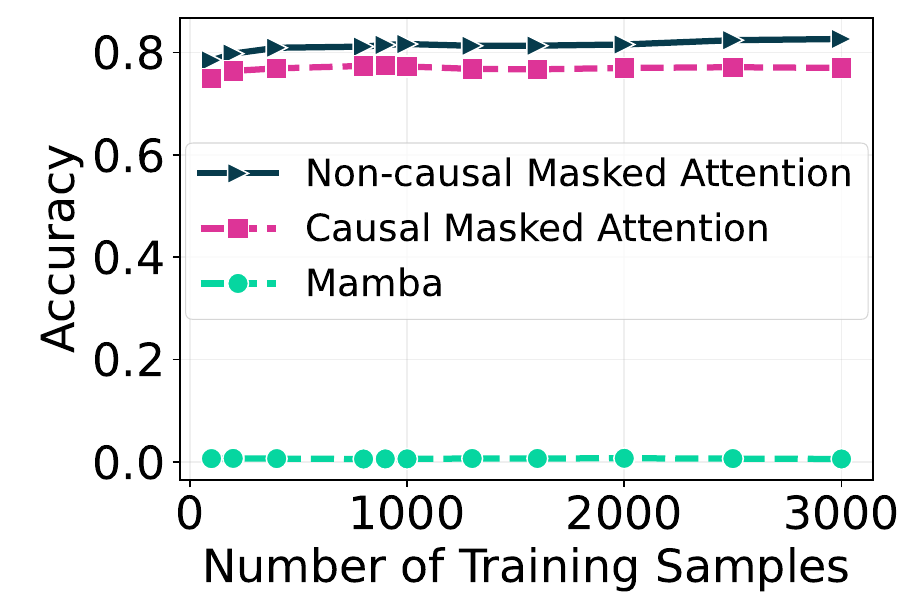}
    \caption{
        \textbf{Effect of causal masking}. When more samples are provided, the non-causal model shows better sample utilization and accuracy, while the causal model's performance plateaus early and declines.
        The Mamba variant performs poorly.
    }
    \end{subfigure}\hfill
    \begin{subfigure}[b]{0.49\textwidth}
        \includegraphics[width=0.9\linewidth]{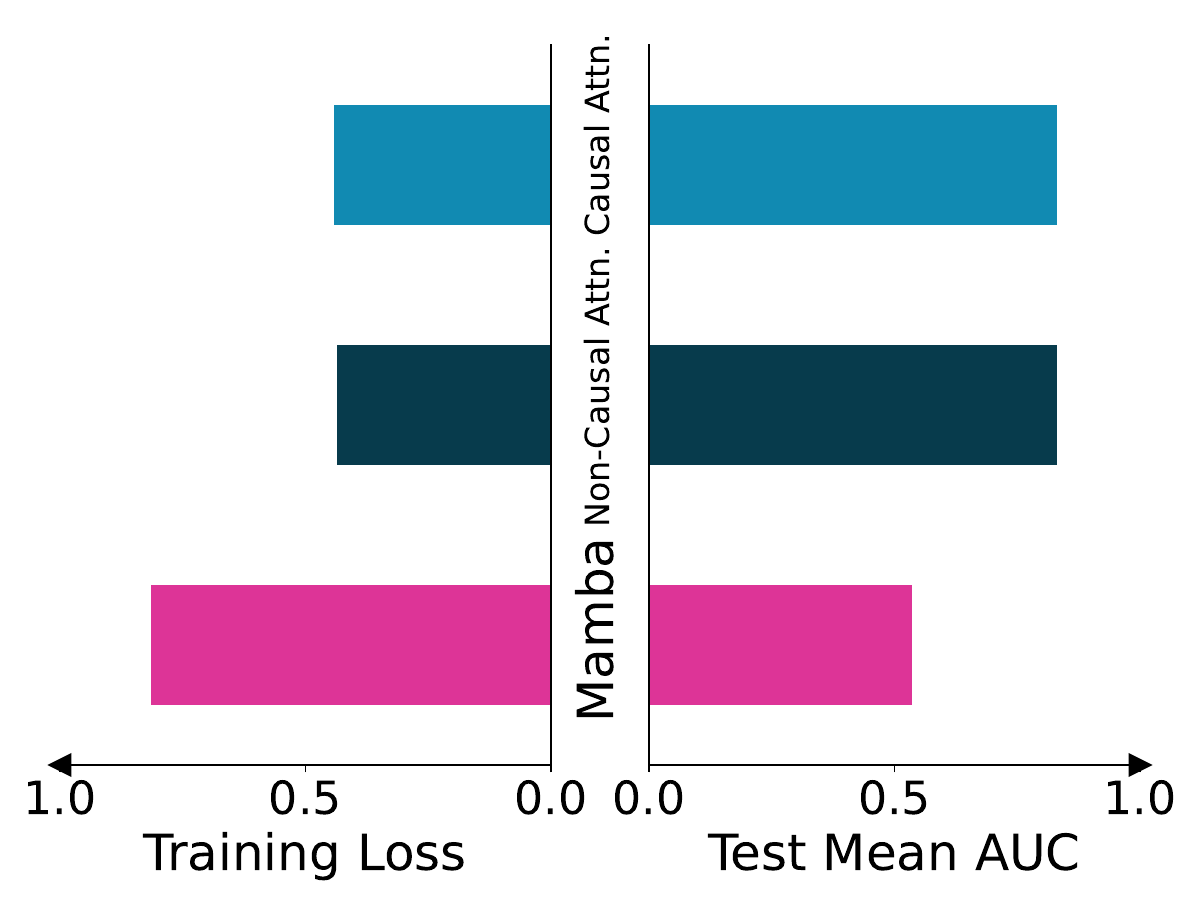}
    \caption{
       \textbf{ICL performance comparison between Mamba and Transformer models}. Results show that Transformer-based models achieve lower training loss and higher AUC across 150 test datasets.
    }
    \end{subfigure}\hfill
    \begin{subfigure}[b]{\textwidth}
        \includegraphics[width=\linewidth]{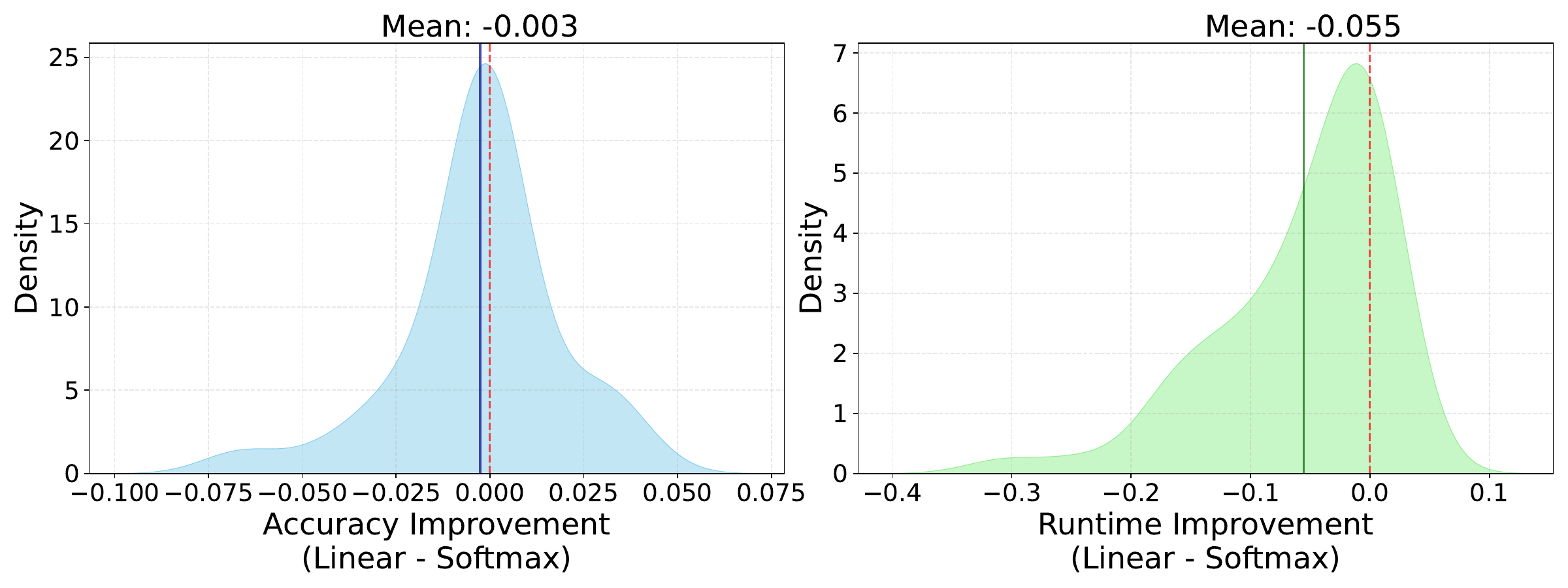}
    \caption{
        \textbf{Accuracy and runtime comparison of softmax and linear attention}. 
        Replacing softmax with linear attention preserves comparable accuracy while significantly reducing runtime.
    }
    \end{subfigure}
    \vspace{-.1in}
    \caption{Impact of model architecture on tabular classification performance.
    Detailed version of Fig.~\ref{fig:attn_explore}.
    }
    \label{fig:attn_explore_detailed}
    \vspace{-.1in}
\end{figure}

\section{Supplement to Section~\ref{sec:ours}: \ours{}}\label{app:ours}
In this section, we provide detailed insights into \ours{}, including training procedures, validation dataset selection, and analyses such as the sensitivity of model selection thresholds on final performance.

\subsection{Computation Analysis of Various Attention Mechanism}\label{app:complexity}
In this part, we provide a computational analysis of various attention mechanisms, comparing standard attention, FlashAttention (specifically FlashAttention-I~\citep{dao2022flashattention}), causal FlashLinearAttention (referred to as FlashLinearAttention in \citet{yang2024gla}), and non-causal linear attention.
To clarify, FlashLinearAttention is designed to reduce HBM access specifically for causal linear attention.
For notational simplicity, we use the term ``linear attention'' to refer to non-causal linear attention.
For both linear attention and FlashLinearAttention, we assume an element-wise kernel feature mapping.
The assumption is reasonable, as popular kernel feature mappings like $\operatorname{elu}(\cdot)+1$ (used here) and $\operatorname{ReLU}(\cdot)$ are also element-wise.
We denote this mapping as $\phi$.

\begin{algorithm}[H]
\SetAlgoLined
\KwIn{Matrices $\mQ, \mK, \mV \in \sR^{N\times D}$ in HBM, on-chip SRAM of size $M$}
Set block size $B$\;
Initialize $\mO = (0)_{N\times D} \in \sR^{N\times D}$ in HBM\;
Divide $\mQ$ into $T = \lceil\frac{N}{B}\rceil$ blocks $\mQ_1, \ldots, \mQ_T$ of size $B\times D$ each, and divide $\mK, \mV$ into $T = \lceil\frac{N}{B}\rceil$ blocks $\mK_1, \ldots, \mK_T$ and $\mV_1 \ldots \mV_T$ of size $B\times D$ each\;
Divide $\mO$ into $T$ blocks $\mO_1, \ldots, \mO_T$ of size $B \times D$ each\;
On on-chip SRAM, construct causal mask, $\mM \in \sR^{B\times B}$\;
On SRAM, initialize $\mS = (0)_{D\times D} \in \sR^{D\times D}$\;
\For{$1 \leq j \leq T$}{
    Load $\mK_j, \mV_j, \mQ_j, \mO_j$ from HBM to on-chip SRAM\;
    On chip, compute $\mK_j \leftarrow \phi(\mK_j)$\;
    On chip, compute $\mQ_j \leftarrow \phi(\mQ_j)$\;
    Write $\mO_j \leftarrow \mQ_j\mS + ((\mQ_j \mK_j^\top) \odot \mM) \cdot \mV_j$ to HBM\;
    On-chip, compute $\mS \leftarrow \mS + \mK_j^\top\mV_j$\;
}
\KwOut{$\mO$}
\caption{Causal FlashLinearAttention Implementation~\citep{yang2024gla}}
\label{alg:fla}
\end{algorithm}

We evaluate these mechanisms based on their High Bandwidth Memory (HBM) access, memory requirements, and floating-point operations (FLOP) when computing attention outputs given query, key, and value inputs.
While \citet{dao2022flashattention} have provided computations for standard attention and FlashAttention, we focus our analysis on causal FlashLinearAttention (detailed in Alg.~\ref{alg:fla}) and HBM-efficient non-causal linear attention (developed by us and detailed in Alg.~\ref{alg:linear_attn}) in Section~\ref{app:hbm_la}.
In practice, we employ a simplified PyTorch implementation of linear attention and demonstrate its efficiency, as it only causes marginal increases in HBM access and memory usage as we demonstrate in Section~\ref{app:pytorch_la}.
Furthermore, we present visualizations in Section~\ref{app:pytorch_la} that illustrate the time and CUDA memory consumption of these attention mechanisms across various sequence lengths and scenarios.

\begin{algorithm}[H]
\SetKwInOut{Input}{Input}
\Input{Matrices $\mQ, \mK, \mV \in \sR^{N\times D}$ in HBM, on-chip SRAM of size $M$}
Set block size $B$\;
Initialize $\mO = (0)_{N\times D} \in \sR^{N\times D}$ in HBM\;
Divide $\mQ$ into $T = \lceil\frac{N}{B}\rceil$ blocks $\mQ_1, \ldots, \mQ_T$ of size $B\times D$ each,
and divide $\mK, \mV$ into $T = \lceil\frac{N}{B}\rceil$ blocks $\mK_1, \ldots, \mK_T$ and $\mV_1, \ldots, \mV_T$ of size $B\times D$ each\;
Divide $\mO$ into $T$ blocks $\mO_1, \ldots, \mO_T$ of size $B \times D$ each\;
On on-chip SRAM, initialize $\mS = (0)_{D\times D} \in \sR^{D\times D}$\;
\For{$1 \leq i \leq T$}{
    Load $\mK_i, \mV_i$\;
    On chip, compute $\mK_i \leftarrow \phi(\mK_i)$ \;
    On chip, compute $\mS \leftarrow \mS + \mK_i^\top\mV_i$\;
}
\For{$1 \leq j \leq T$}{
    Load $\mQ_j, \mO_j$\;
    On chip, compute $\mQ_j \leftarrow \phi (\mQ_j)$ \;
    Write $\mO_j \leftarrow \mQ_j\mS$ to HBM\;
}
\KwOut{$\mO$}
\caption{HBM-Efficient Implementation of Linear Attention}
\label{alg:linear_attn}
\end{algorithm}

\subsubsection{HBM-Efficient Linear Attention}\label{app:hbm_la}
Here, we analyze the number of HBM accesses, HBM memory, and FLOPS required by FlashLinearAttention (Alg.~\ref{alg:fla}) and linear attention (Alg.~\ref{alg:linear_attn}).

\begin{lemma}\label{lemma:hbm_access}
    Let $\mQ, \mK, \mV \in \sR^{N \times D}$ represent the query, key, and value matrices for a single attention head, where $N$ is the sequence length and $D$ is the embedding size. 
    Both FlashLinearAttention (Alg.~\ref{alg:fla}) and linear attention (Alg.~\ref{alg:linear_attn}) require $5ND$ HBM accesses to compute the attention output. 
\end{lemma}
\begin{proof}[Proof of Lemma~\ref{lemma:hbm_access}]
For causal FlashLinearAttention (Alg.~\ref{alg:fla}):
\begin{itemize}[leftmargin=*]
    \item Line 8: Loading $\mK_j, \mV_j, \mQ_j, \mO_j$ necessitates $4BD$ HBM accesses.
    \item Line 11: Writing $\mO_j$ requires $BD$ HBM accesses.
\end{itemize}
These operations are executed $T$ times, where $T = \lceil \frac{N}{B} \rceil$. Thus, the total HBM accesses are:
\[
5BD \cdot T = 5BD \cdot \lceil \frac{N}{B} \rceil = 5ND.
\]

For non-causal linear attention (Alg.~\ref{alg:linear_attn}):
\begin{itemize}[leftmargin=*]
    \item Line 7: Loading $\mK_i, \mV_i$ requires $2BD$ HBM accesses.
    \item Line 11: Loading $\mQ_j, \mO_j$ requires $2BD$ HBM accesses.
    \item Line 13: Writing $\mO_j$ requires $BD$ HBM accesses.
\end{itemize}
These operations are also repeated $T$ times, where $T = \lceil \frac{N}{B} \rceil$. Consequently, the total HBM accesses are:
\[
5BD \cdot T = 5BD \cdot \lceil \frac{N}{B} \rceil = 5ND.
\]

Therefore, we conclude that both causal FlashLinearAttention and non-causal linear attention require $5ND$ HBM accesses to compute the attention output.
\end{proof}

\begin{lemma}\label{lemma:hbm_memory}
     Let $\mQ, \mK, \mV \in \sR^{N \times D}$ represent the query, key, and value matrices for a single attention head, where $N$ is the sequence length and $D$ is the embedding size. 
    Both FlashLinearAttention (Alg.~\ref{alg:fla}) and linear attention (Alg.~\ref{alg:linear_attn}) require $4ND$ HBM memory to compute the attention output. 
\end{lemma}
\begin{proof}[Proof of Lemma~\ref{lemma:hbm_memory}]
For both algorithms:
\begin{itemize}[leftmargin=*]
    \item Storing $\mQ, \mK, \mV$ requires $3ND$ memory.
    \item Storing $\mO$ requires $ND$ memory.
\end{itemize}
Total HBM memory usage: $4ND$.
\end{proof}

\begin{lemma}\label{lemma:flops}
     Let $\mQ, \mK, \mV \in \sR^{N \times D}$ represent the query, key, and value matrices for a single attention head, where $N$ is the sequence length and $D$ is the embedding size. 
    Both FlashLinearAttention (Alg.~\ref{alg:fla}) and linear attention (Alg.~\ref{alg:linear_attn}) require $O(ND^2)$ FLOPS to compute the attention output. 
\end{lemma}
\begin{proof}[Proof of Lemma~\ref{lemma:flops}]
For causal FlashLinearAttention (Alg.~\ref{alg:fla}):
\begin{itemize}[leftmargin=*]
    \item Computing $\phi(\mK_j)$ and $\phi(\mQ_j)$ requires $2BD$ FLOPs.
    \item Computing $(\mQ_j\mK_j^\top) \odot \mathbf{M}$ requires $B^2(2D-1) + B^2$ FLOPs.
    \item The result of last step multiplied by $\mV_j$ requires $B^2(2D-1) + BD(2B-1)$ FLOPs.
    \item Computing $\mQ_j\mS$ requires $B \cdot D(2D-1)$ FLOPs.
    \item The addition of the last two steps requires $BD$ FLOPs.
    \item Computing $\mK_j^\top\mV_j$ (line 10) requires $(2B-1) \cdot D^2$ FLOPs.
    \item The addition of $\mS$ and the last step requires $D^2$ FLOPs.
\end{itemize}
The total number of FLOPs for one iteration is:
\begin{align}
    &2BD + B^2(2D-1) + B^2 + B^2(2D-1) + BD(2B-1) + B \cdot D(2D-1) + BD + (2B-1) \cdot D^2 + D^2\\
    &= 4B^2D + 2BD + 4BD^2.
\end{align}
These operations are repeated $T = \lceil \frac{N}{B} \rceil$ times. The total number of FLOPs is:
\begin{align}
    (4B^2D + 2BD + 4BD^2) \cdot T = O(ND^2).
\end{align}

For non-causal linear attention (Alg.~\ref{alg:linear_attn}):
\begin{itemize}[leftmargin=*]
    \item Computing $\phi(\mK_i)$ requires $BD$ FLOPs.
    \item Computing $\mS + \mK_i^\top \mV_i$ (line 9) requires $2BD^2$ FLOPs.
    \item Computing $\phi(\mQ_j)$ (line 12) requires $BD$ FLOPS.
    \item Computing $\mQ_j\mS$ (line 13) requires $(2D-1)BD$ FLOPs.
\end{itemize}
These operations are repeated $T = \lceil \frac{N}{B} \rceil$ times. The total number of FLOPs is:
\begin{align}
    (BD + 4BD^2) \cdot T = O(ND^2).
\end{align}

Thus, we conclude that both algorithms require $O(ND^2)$ FLOPs to compute the attention output.
\end{proof}

\subsubsection{Simplified PyTorch Implementation of Linear Attention}\label{app:pytorch_la}
In our implementation, we adopt a straightforward PyTorch approach to linear attention rather than an HBM-efficient method.
We employ the concise two-line implementation presented in Listing~\ref{listing:linear_attn}.
In the following lemma, we demonstrate that this straightforward implementation only incurs a marginal increase in HBM accesses and HBM memory usage.

\begin{lstlisting}[language=Python, caption=Straightforward PyTorch implementation of linear attention~\citep{katharopoulos2020transformers}., label={listing:linear_attn}]
def linear_attn(q, k, v): 
    """
    q: (batch, heads, seq_q, dim_qk)
    k: (batch, heads, seq_kv, dim_qk)
    v: (batch, heads, seq_kv, dim_v)
    """
    kv = torch.einsum("bhnd,bhnm->bhdm", k, v)
    o = torch.einsum("bhld,bhdm->bhlm", q, kv)
    return o.contiguous()
\end{lstlisting}

\setcounter{temp_theorem_counter}{\value{theorem}} 
\setcounter{theorem}{\numexpr\getrefnumber{thm:computation}-1\relax} 
\begin{theorem}
     Let $\mQ, \mK, \mV \in \sR^{N \times D}$ represent the query, key, and value matrices for a single attention head, where $N$ is the sequence length and $D$ is the embedding size. 
    Both causal FlashLinearAttention (Alg.~\ref{alg:fla}) and non-causal linear attention (Listing~\ref{listing:linear_attn}) require $O(ND)$ HBM accesses, $O(ND)$ HBM memory, and $O(ND^2)$ FLOPS to compute the attention output. 
\end{theorem}
\setcounter{theorem}{\value{temp_theorem_counter}}
\begin{proof}
Let us consider the implementation in Listing~\ref{listing:linear_attn} and compare it to Alg.~\ref{alg:linear_attn}.
PyTorch's optimized tensor computation ensures efficiency, with the primary distinction between Listing~\ref{listing:linear_attn} and Alg.~\ref{alg:linear_attn} being the storage of \texttt{kv} in the former, which is equivalent to $\mS \in \sR^{D\times D}$ in Alg.~\ref{alg:linear_attn}.
This results in the following changes:
\begin{itemize}[leftmargin=*]
\item HBM Accesses: By Lemma~\ref{lemma:hbm_access}, Alg.~\ref{alg:linear_attn} requires $5ND$ HBM accesses.
Due to the additional write and load operations for $\mS\in\sR^{D\times D}$, Listing~\ref{listing:linear_attn} requires $5ND+2D^2$ HBM accesses.
\item HBM Memory Usage: By Lemma~\ref{lemma:hbm_memory}, Alg.~\ref{alg:linear_attn} requires $4ND$ HBM memory usage.
Due to the additional storage requirements for $\mS\in\sR^{D\times D}$, Listing~\ref{listing:linear_attn} requires $4ND+D^2$ HBM memory usage.
\end{itemize}
The number of FLOPS remains unaffected.
The analysis above, in conjunction with Lemmas~\ref{lemma:hbm_access}, \ref{lemma:hbm_memory}, and \ref{lemma:flops}, yields the desired outcome.
\end{proof}

In Table~\ref{tab:complexity}, we summarize the \#HBM access, HBM memory, and FLOPS required by standard attention (with naive PyTorch implementation), FlashAttention-I, FlashLinearAttention (causal), and linear attention with both implementations.
\begin{table}[htbp]
\centering
\resizebox{\textwidth}{!}{
\begin{tabular}{l|ccccc}
\toprule
 & \multirow{2}{*}{\textbf{Standard Attention}} & \multirow{1}{*}{\textbf{FlashAttention}} & \multirow{1}{*}{\textbf{FlashLinearAttention}} & \multicolumn{2}{c}{\textbf{Linear Attention}} \\
 & & \citep{dao2022flashattention} & \citep{yang2024gla} & Alg.~\ref{alg:linear_attn} & Listing~\ref{listing:linear_attn}\\ 
\midrule
\# HBM access & $4N^2 + 4ND$ & $\frac{12N^2D^2}{M} + \frac{16N^2D}{M} + 2ND$ & $5ND$ & $5ND$ & $5ND+2D^2$\\
\midrule
Memory & $2N^2 + 4ND$ & $2N + 4ND$ & $4ND$ & $4ND$ & $4ND+D^2$ \\
\midrule
FLOPS & $O(N^2D)$ & $O(N^2D)$ & $O(ND^2)$ & $O(ND^2)$ & $O(ND^2)$ \\
\bottomrule
\end{tabular}}
\vspace{.1in}
\caption{
    \textbf{Comparison of memory and computational costs across different attention mechanisms. }
    FlashAttention improves the speed of standard attention by optimizing \# HBM access. 
    Flash causal linear attention takes a similar approach, achieving linear \# HBM access. 
    However, we show that non-causal linear attention already achieves linear \# HBM access, matching the efficiency of flash causal linear attention without requiring any additional optimization on \# HBM access.
}
\label{tab:complexity}
\end{table}

Subsequently, we visualize the empirical execution time and CUDA memory utilization of FlashAttention-2, FlashLinearAttention, and linear attention in Fig.~\ref{fig:time_attn} and Fig.~\ref{fig:memory_attn}, respectively. 
We vary the head dimension $\in \set{32, 64, 128, 256}$, the number of heads $\in \set{2, 4, 8, 16}$, and the sequence length $\in \set{2^4, 2^5, \ldots, 2^{15}}$. 
We focus on the self-attention case, randomly generating input (serving as key, query, and values) with a batch size of 10, and replicate the experiment 5 times.
The final values presented are aggregated across these 5 simulations.
Notably, we were unable to obtain results for FlashLinearAttention in two configurations: (1) head dimension 256 with 8 heads, and (2) head dimension 256 with 16 heads, due to illegal memory access error incurred by the PyTorch package \texttt{fla}~\citep{yang2024gla}.
Our observations from the figures indicate that both runtime and CUDA memory usage of FlashLinearAttention and linear attention exhibit linear growth with respect to sequence length, aligning with the predictions of Theorem~\ref{thm:computation}.

\begin{figure}[htbp]
    \centering
    \begin{subfigure}[b]{0.49\textwidth}
    \includegraphics[width=\linewidth]{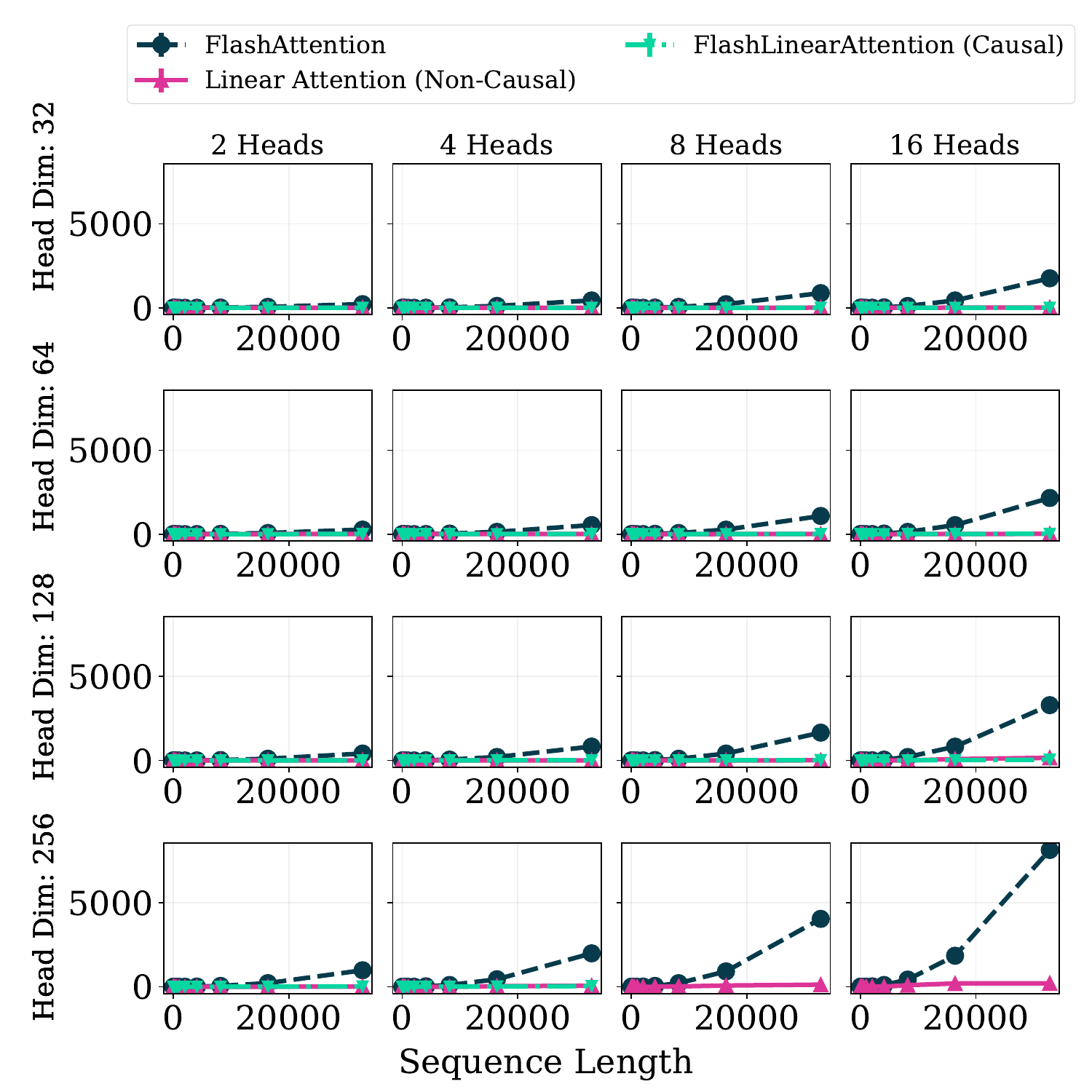}
    \caption{Time}
    \label{fig:time_attn}
    \end{subfigure}
    \hfill
    \begin{subfigure}[b]{0.49\textwidth}
    \includegraphics[width=\linewidth]{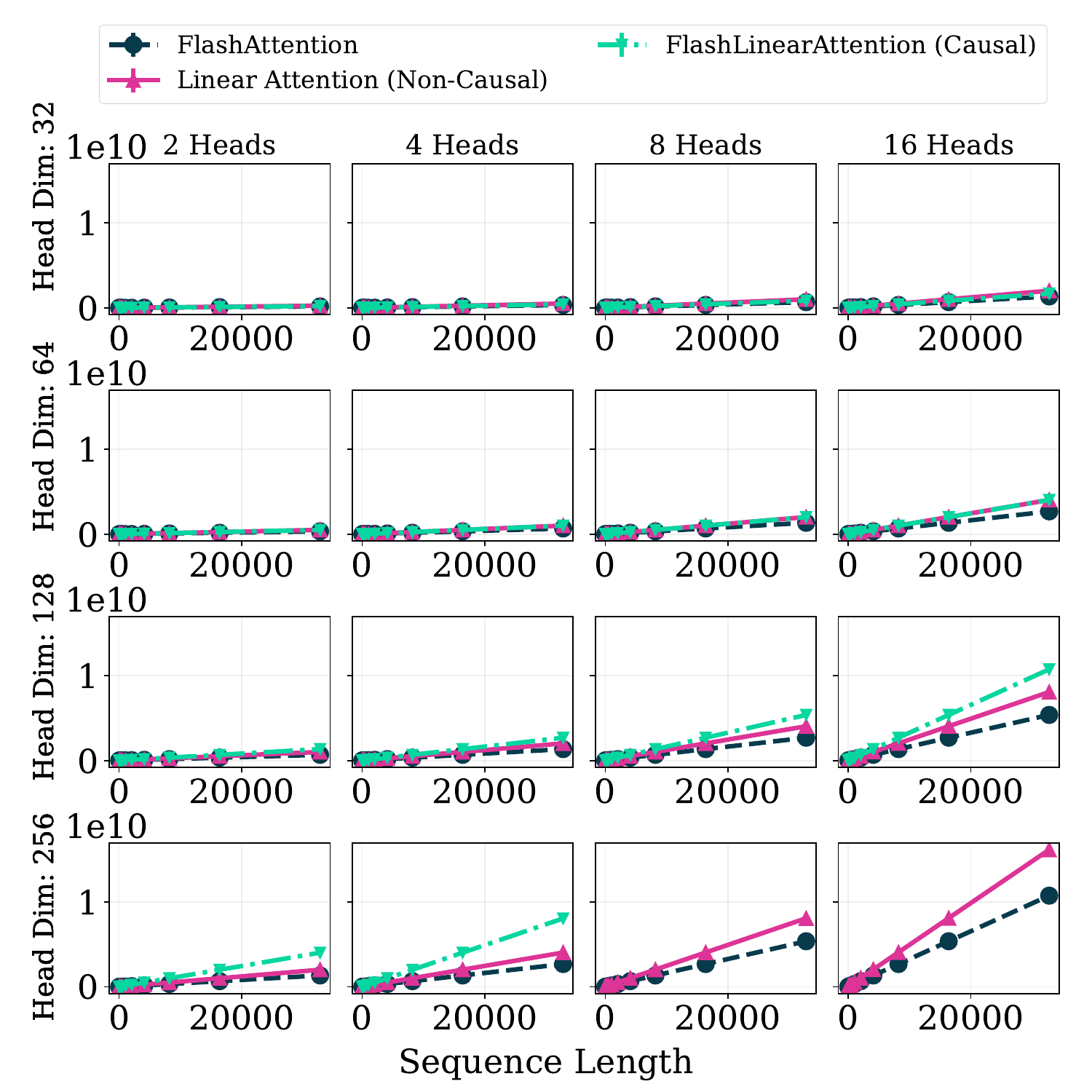}
    \caption{CUDA Memory}
    \label{fig:memory_attn}
    \end{subfigure}
   \caption{Time and CUDA memory usage comparison of FlashAttention-2~\citep{dao2024flashattention}, causal FlashLinearAttention~\citep{yang2024gla}, and linear attention~\citep{katharopoulos2020transformers} (implemented as in Listing~\ref{listing:linear_attn}).
   Results for FlashLinearAttention in two configurations: (1) head dimension 256 with 8 heads, and (2) head dimension 256 with 16 heads are missing, due to illegal memory access error incurred by the PyTorch package \texttt{fla}~\citep{yang2024gla}.
   }
\end{figure}

\subsection{Model Training}\label{app:train}
We implement linear attention with the feature function $\texttt{elu}(\cdot)+1$, adhering to the default implementation proposed by \citet{katharopoulos2020transformers}.
Unless otherwise specified, we adopt the training setup of \tabpfn{} for \ours{}-S100, \ours{}-L100, and \ours{}-H1K. Each model is trained on a single Nvidia A100 80GB PCIe GPU.

\begin{table}[htbp]
    \centering
    \begin{tabular}{c|cccc}
    \toprule
         \textbf{Hyperparameters} & \textbf{Batch Size}  & \textbf{Epoch} & \textbf{Learning Rate} & \textbf{\#Steps\slash epoch} \\
         \midrule
        \ours{}-S100 & 1210 & 8 & 3e-5 & 8192\\
        \ours{}-L100 & 110 & 4 & 3e-5 & 8192 \\
        \ours{}-H1K & 1410 & 4 & 3e-5 & 1024 \\
        \bottomrule
    \end{tabular}%
    \caption{
    \textbf{Hyperparameters used for training \ours{} models.}
    The number of steps per epoch indicates the quantity of synthetic datasets generated and used for training within each epoch.
    }
    \label{tab:hyperparameters}
\end{table}

\begin{wrapfigure}{L}{0.5\textwidth}
    \centering
    \includegraphics[width=\linewidth]{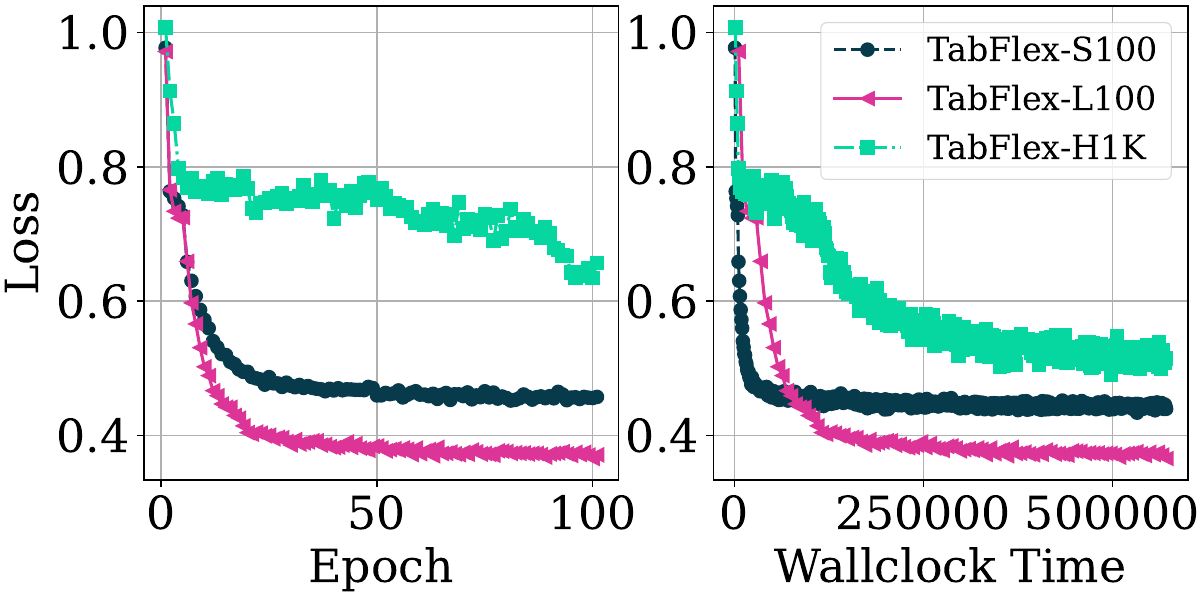}
    \caption{Visualization of training loss for \ours{} models as a function of epoch and wallclock time.
    }
    \label{fig:loss}
    \vspace{-.2in}
\end{wrapfigure}
Table~\ref{tab:hyperparameters} summarizes the hyperparameters selected for training \ours{}-S100, \ours{}-L100, and \ours{}-H1K.
For all three methods, we utilize the same embedding size of 512, consistent with \tabpfn{}.
We extend the feature capacity by modifying the first linear layer, which projects the features into embeddings -- specifically, we increase the number of neurons responsible for receiving the features.

The training loss curves are illustrated in Fig.~\ref{fig:loss}.
We observe that as the number of features and the length of training dataset sequences increase, the training process becomes more time-consuming.
In fact, training a robust \ours{}-H1K model requires more than three weeks.

\subsection{Validation Datasets}\label{app:val}
We select the validation datasets from the OpenML AutoML Benchmark~\citep{automl} by choosing 10 datasets from each of the following sample size intervals: $[0.1\text{K}, 1\text{K})$, $[1\text{K}, 10\text{K})$, and $[10\text{K}, 100\text{K})$. To ensure diversity in the validation set, we also vary the number of classes and features within each interval. The details of all datasets used in validation are summarized in Table~\ref{tab:val}.
\begin{table}[htbp]
    \centering
    \resizebox{0.9\textwidth}{!}{
\begin{tabular}{lp{3in}|ccc}
\toprule
\textbf{OpenML did} & \textbf{Dataset} & \textbf{\#Features} & \textbf{\#Instances} & \textbf{\#Classes} \\
\midrule
279 & meta-stream-intervals.arff & 75 & 45164 & 11 \\
311 & oil-spill & 50 & 937 & 2 \\
742 & fri-c4-500-100 & 101 & 500 & 2 \\
825 & boston-corrected & 21 & 506 & 2 \\
833 & bank32nh & 33 & 8192 & 2 \\
841 & stock & 10 & 950 & 2 \\
920 & fri-c2-500-50 & 51 & 500 & 2 \\
940 & water-treatment & 37 & 527 & 2 \\
981 & kdd-internet-usage & 69 & 10108 & 2 \\
1039 & hiva-agnostic & 1618 & 4229 & 2 \\
1491 & one-hundred-plants-margin & 65 & 1600 & 100 \\
1492 & one-hundred-plants-shape & 65 & 1600 & 100 \\
1503 & spoken-arabic-digit & 15 & 263256 & 10 \\
1515 & micro-mass & 1301 & 571 & 20 \\
1536 & volcanoes-b6 & 4 & 10130 & 5 \\
1541 & volcanoes-d4 & 4 & 8654 & 5 \\
1549 & autoUniv-au6-750 & 41 & 750 & 8 \\
40645 & GAMETES-Epistasis-2-Way-1000atts-0.4H-EDM-1-EDM-1-1 & 1001 & 1600 & 2 \\
40672 & fars & 30 & 100968 & 8 \\
40677 & led24 & 25 & 3200 & 10 \\
40693 & xd6 & 10 & 973 & 2 \\
40705 & tokyo1 & 45 & 959 & 2 \\
40922 & Run-or-walk-information & 7 & 88588 & 2 \\
40985 & tamilnadu-electricity & 4 & 45781 & 20 \\
41082 & USPS & 257 & 9298 & 10 \\
41144 & madeline & 260 & 3140 & 2 \\
41986 & GTSRB-HOG01 & 1569 & 51839 & 43 \\
41988 & GTSRB-HOG02 & 1569 & 51839 & 43 \\
41989 & GTSRB-HOG03 & 2917 & 51839 & 43 \\
41990 & GTSRB-HueHist & 257 & 51839 & 43 \\
41991 & Kuzushiji-49 & 785 & 270912 & 49 \\
42193 & compas-two-years & 14 & 5278 & 2 \\
42206 & porto-seguro & 38 & 595212 & 2 \\
42343 & KDD98 & 478 & 82318 & 2 \\
\bottomrule
\end{tabular}}
    \caption{Characteristics of datasets in our diverse validation set.}
    \label{tab:val}
\end{table}

\subsection{Sensitivity Analysis on Decision Boundaries}\label{app:sensitivity}
The decision thresholds align with the training regimes of the models. 
\ours{}-S100, sharing TabPFN’s training setup but with an updated architecture, is deployed similarly $(n \leq 3K, d \leq 100)$. 
\ours{}-L100, trained on low-dimensional $(d \leq 100)$ but larger datasets, is used for longer sequences $(n \geq 3K, d \leq 100)$. \ours{}-H1K, trained on high-dimensional data, is assigned to handle those cases accordingly.

We note that performance is not highly sensitive to the chosen decision boundaries. 
To demonstrate this, we conducted additional experiments on simple (low-dimension, small-size), low-dimensional \& large, and high-dimensional \& large datasets—two datasets per setting—and present the results for all three models in Table~\ref{tab:robustness}.

\begin{table}[]
\centering
\resizebox{0.85\textwidth}{!}{
\begin{tabular}{ccccccc}
\toprule
\multicolumn{2}{c}{\textbf{Dataset}}                                 & \multicolumn{1}{l}{\textbf{Metric}} & \multicolumn{1}{l}{\textbf{TabPFN}} & \multicolumn{1}{l}{\textbf{TabFlex-S100}} & \multicolumn{1}{l}{\textbf{TabFlex-L100}} & \multicolumn{1}{l}{\textbf{TabFlex-H1K}} \\
\midrule
\multirow{4}{*}{\textbf{Simple}}                    & \multirow{2}{*}{credit-g}       & Accuracy                             & 0.79 & \textbf{0.82} & 0.79 & 0.75                              \\
&                & Time (s) & 0.23 & 0.13 & 0.13 & 0.13 \\ \cmidrule{2-7}
& \multirow{2}{*}{diabet}         & AUC  & \textbf{0.78} & \textbf{0.78} & 0.77 & \textbf{0.78}  \\
&                & Time (s)   & 0.15 & 0.08 & 0.10 & 0.09 \\
\midrule
\multirow{4}{*}{\textbf{Low-Dimensional \& Large}}  & \multirow{2}{*}{bank-marketing} & AUC   & 0.89 & 0.89 & \textbf{0.90} & 0.89\\
&                & Time (s) & 1.75 & 0.25 & 2.43 & 1.67 \\
\cmidrule{2-7}
& \multirow{2}{*}{elevators}      & AUC  & 0.94 & 0.94 & \textbf{0.95} & 0.94 \\
&                & Time (s) & 1.11 & 0.22 & 0.7 & 0.7 \\
\midrule
\multirow{4}{*}{\textbf{High-dimensional \& Large}} & \multirow{2}{*}{nomao}          & AUC & 0.86 & 0.83 & 0.75 & \textbf{0.99}\\
&                & Time (s) & 1.95 & 0.86 & 4.71 & 4.63 \\
\cmidrule{2-7}
& \multirow{2}{*}{SpeedDating}    & AUC  & 0.66 & 0.69 & 0.59 & \textbf{0.83}\\
&                & Time (s) & 2.86 & 0.89 & 1.63 & 1.71\\ \bottomrule                               
\end{tabular}
}
\caption{
Performance of \ours{}-S100, \ours{}-L100, and \ours{}-H1K across three types of datasets, using TabPFN as a baseline. We observe that all \ours{} variants perform well on both simple and low-dimensional large datasets, demonstrating that performance is fairly robust to the choice of decision threshold.
}
\label{tab:robustness}
\end{table}

\section{Supplement to Section~\ref{sec:exp}: Performance Evaluation of \ours{}}\label{app:exp}

\subsection{TabZilla Datasets}\label{app:tabzilla}
The results of our experiments on TabZilla-related datasets are reported in Table~\ref{tab:tabzilla1}, \ref{tab:tabzilla2}, and \ref{tab:tabzilla4}.
\citep{mcelfresh2023tabzilla} presents the details of the datasets used in their hard benchmark (Table~\ref{tab:tabzilla4}) in Table 4 of their paper.
We provide the specifications of the datasets used for our evaluation in Table~\ref{tab:tabzilla1} and Table~\ref{tab:tabzilla2} in Table~\ref{tab:dataset_tabzilla1} and Table~\ref{tab:dataset_tabzilla2}, respectively.
\begin{table}[htbp]
    \centering
    \resizebox{.8\textwidth}{!}{
    \begin{tabular}{lrrr|lrrr|p{1in}rrr}
\toprule
\textbf{Dataset} & $D$ & $N$ & $C$ & \textbf{Dataset} & $D$ & $N$ & $C$ & \textbf{Dataset} & $D$ & $N$ & $C$ \\
\midrule
cmc & 9 & 1473 & 3 & socmob & 5 & 1156 & 1 & adult-census & 14 & 32561 & 2 \\ \midrule
kc1 & 21 & 2109 & 1 & vehicle & 18 & 846 & 4 & breast-cancer & 9 & 286 & 2 \\ \midrule
kc2 & 21 & 522 & 1 & heart-h & 13 & 294 & 1 & mfeat-factors & 216 & 2000 & 10 \\ \midrule
pc3 & 37 & 1563 & 1 & jasmine & 144 & 2984 & 1 & mfeat-zernike & 47 & 2000 & 10 \\ \midrule
pc4 & 37 & 1458 & 1 & phoneme & 5 & 5404 & 1 & dresses-sales & 12 & 500 & 2 \\ \midrule
pc1 & 21 & 1109 & 1 & semeion & 256 & 1593 & 10 & mfeat-fourier & 76 & 2000 & 10 \\ \midrule
cjs & 33 & 2796 & 6 & heart-c & 13 & 303 & 1 & balance-scale & 4 & 625 & 3 \\ \midrule
car & 6 & 1728 & 4 & kr-vs-kp & 36 & 3196 & 1 & bank-marketing & 16 & 45211 & 2 \\ \midrule
tae & 5 & 151 & 3 & spambase & 57 & 4601 & 1 & car-evaluation & 21 & 1728 & 4 \\ \midrule
jm1 & 21 & 10885 & 1 & satimage & 36 & 6430 & 6 & cylinder-bands & 37 & 540 & 2 \\ \midrule
dna & 180 & 3186 & 3 & mushroom & 22 & 8124 & 1 & mfeat-karhunen & 64 & 2000 & 10 \\ \midrule
musk & 167 & 6598 & 1 & diabetes & 8 & 768 & 1 & credit-approval & 15 & 690 & 2 \\ \midrule
wdbc & 30 & 569 & 1 & rabe\_266 & 2 & 120 & 1 & ozone-level-8hr & 72 & 2534 & 2 \\ \midrule
wilt & 5 & 4839 & 1 & breast-w & 9 & 699 & 1 & analcatdata\_dmft & 4 & 797 & 6 \\ \midrule
ilpd & 10 & 583 & 1 & elevators & 18 & 16599 & 1 & monks-problems-2 & 6 & 601 & 2 \\ \midrule
sick & 28 & 3772 & 1 & Satellite & 36 & 5100 & 1 & cardiotocography & 35 & 2126 & 10 \\ \midrule
iris & 4 & 150 & 3 & fertility & 9 & 100 & 1 & PhishingWebsites & 30 & 11055 & 2 \\ \midrule
lymph & 18 & 148 & 4 & ionosphere & 34 & 351 & 1 & synthetic\_control & 60 & 600 & 6 \\ \midrule
churn & 20 & 5000 & 1 & transplant & 3 & 131 & 1 & steel-plates-fault & 27 & 1941 & 7 \\ \midrule
colic & 22 & 368 & 1 & eucalyptus & 19 & 736 & 5 & mfeat-morphological & 6 & 2000 & 10 \\ \midrule
ecoli & 7 & 336 & 8 & Australian & 14 & 690 & 1 & acute-inflammations & 6 & 120 & 2 \\ \midrule
autos & 25 & 205 & 6 & hayes-roth & 4 & 160 & 3 & analcatdata\_boxing1 & 3 & 120 & 2 \\ \midrule
scene & 299 & 2407 & 1 & dermatology & 34 & 366 & 6 & analcatdata\_chlamydia & 3 & 100 & 2 \\ \midrule
profb & 9 & 672 & 1 & MiceProtein & 77 & 1080 & 8 & wall-robot-navigation & 24 & 5456 & 4 \\ \midrule
colic & 26 & 368 & 1 & SpeedDating & 120 & 8378 & 1 & visualizing\_livestock & 2 & 130 & 2 \\ \midrule
labor & 16 & 57 & 1 & tic-tac-toe & 9 & 958 & 1 & Click\_prediction\_small & 11 & 39948 & 2 \\ \midrule
irish & 5 & 500 & 1 & hill-valley & 100 & 1212 & 1 & analcatdata\_authorship & 70 & 841 & 4 \\ \midrule
glass & 9 & 214 & 6 & page-blocks & 10 & 5473 & 5 & banknote-authentication & 4 & 1372 & 2 \\ \midrule
yeast & 8 & 1269 & 4 & lung-cancer & 56 & 32 & 3 & LED-display-domain-7digit & 7 & 500 & 10 \\ \midrule
sonar & 60 & 208 & 1 & qsar-biodeg & 41 & 1055 & 1 & visualizing-environmental & 3 & 111 & 2 \\ \midrule
splice & 60 & 3190 & 3 & fri\_c3\_100\_5 & 5 & 100 & 1 & postoperative-patient-data & 8 & 88 & 2 \\ \midrule
libras & 104 & 360 & 10 & ada\_agnostic & 48 & 4562 & 1 & blood-transfusion-service-center & 4 & 748 & 2 \\ \midrule
anneal & 38 & 898 & 5 & fri\_c0\_100\_5 & 5 & 100 & 1 & \\ 
\bottomrule
\end{tabular}}
    \caption{Datasets utilized in the evaluation presented in Table~\ref{tab:tabzilla1}.
    Here $D$, $N$, and $C$ denote the number of features, instances, and classes, respectively.
    }
    \label{tab:dataset_tabzilla1}
\end{table}

\begin{table}[htbp]
    \centering
    \resizebox{.6\textwidth}{!}{
\begin{tabular}{l|ccc}
\toprule
 \textbf{Dataset} & \textbf{\#Features} & \textbf{\#Instances} & \textbf{\#Classes} \\
\midrule
Australian & 14 & 690 & 2 \\
LED-display-domain-7digit & 7 & 500 & 10 \\
MiceProtein & 77 & 1080 & 8 \\
acute-inflammations & 6 & 120 & 2 \\
analcatdata\_authorship & 70 & 841 & 4 \\
analcatdata\_boxing1 & 3 & 120 & 2 \\
analcatdata\_chlamydia & 3 & 100 & 2 \\
analcatdata\_dmft & 4 & 797 & 6 \\
anneal & 38 & 898 & 5 \\
autos & 25 & 205 & 6 \\
balance-scale & 4 & 625 & 3 \\
blood-transfusion-service-center & 4 & 748 & 2 \\
blood-transfusion-service-center & 4 & 748 & 2 \\
breast-cancer & 9 & 286 & 2 \\
breast-w & 9 & 699 & 2 \\
colic & 26 & 368 & 2 \\
colic & 22 & 368 & 2 \\
credit-approval & 15 & 690 & 2 \\
cylinder-bands & 37 & 540 & 2 \\
dermatology & 34 & 366 & 6 \\
diabetes & 8 & 768 & 2 \\
dresses-sales & 12 & 500 & 2 \\
ecoli & 7 & 336 & 8 \\
eucalyptus & 19 & 736 & 5 \\
fertility & 9 & 100 & 2 \\
fri\_c0\_100\_5 & 5 & 100 & 2 \\
fri\_c3\_100\_5 & 5 & 100 & 2 \\
glass & 9 & 214 & 6 \\
hayes-roth & 4 & 160 & 3 \\
heart-c & 13 & 303 & 2 \\
heart-h & 13 & 294 & 2 \\
hill-valley & 100 & 1212 & 2 \\
ilpd & 10 & 583 & 2 \\
ionosphere & 34 & 351 & 2 \\
iris & 4 & 150 & 3 \\
irish & 5 & 500 & 2 \\
kc2 & 21 & 522 & 2 \\
labor & 16 & 57 & 2 \\
lung-cancer & 56 & 32 & 3 \\
lymph & 18 & 148 & 4 \\
monks-problems-2 & 6 & 601 & 2 \\
pc1 & 21 & 1109 & 2 \\
postoperative-patient-data & 8 & 88 & 2 \\
profb & 9 & 672 & 2 \\
qsar-biodeg & 41 & 1055 & 2 \\
rabe\_266 & 2 & 120 & 2 \\
socmob & 5 & 1156 & 2 \\
sonar & 60 & 208 & 2 \\
synthetic\_control & 60 & 600 & 6 \\
tae & 5 & 151 & 3 \\
tic-tac-toe & 9 & 958 & 2 \\
transplant & 3 & 131 & 2 \\
vehicle & 18 & 846 & 4 \\
visualizing\_environmental & 3 & 111 & 2 \\
visualizing\_livestock & 2 & 130 & 2 \\
wdbc & 30 & 569 & 2 \\
yeast & 8 & 1269 & 4 \\
\bottomrule
\end{tabular}}
    \caption{Datasets utilized in the evaluation presented in Table~\ref{tab:tabzilla2}.
    }
    \label{tab:dataset_tabzilla2}
\end{table}

\subsection{Extended Experiment Setup}
\label{App:experiment_setup}
\paragraph{Baselines.}
We evaluate our approach against a comprehensive set of baselines, as considered by \citet{mcelfresh2023tabzilla}. 
These include:
(i) classical methods: Random Forest \citep{liaw2002classification}, SVM \citep{cortes1995support}, LinearModel \citep{cox1958regression}, KNN \citep{cover1967nearest} and Decision Tree \citep{quinlan1986induction};
(ii) Gradient Boosted Decision Trees (GBDT) methods: XGBoost \citep{chen2016xgboost}, CatBoost \citep{prokhorenkova2018catboost}, and LightGBM \citep{ke2017lightgbm};
(iii) Non-Transformer Neural Network (Non-TF NN) methods: SAINT \citep{somepalli2021saint}, ResNet \citep{resnet}, DANet \citep{chen2022danets}, NODE \citep{popov2019neural}, MLP~\citep{mlp}, MLP-rtdl \citep{gorishniy2021revisiting}, DeepFM \citep{guo2017deepfm}, STG \citep{yamada2020feature}, VIME \citep{yoon2020vime}, and NAM \citep{agarwal2021neural};
(iv) Transformer (TF) methods: \tabpfn{} \citep{hollmann2023tabpfn}, FTTransformer \citep{gorishniy2021revisiting}, TabNet \citep{arik2021tabnet}, and TabTransformer \citep{huang2020tabtransformer}.
The results for these methods, except \tabpfn{}, are taken directly from \citet{mcelfresh2023tabzilla}, who conducted their experiments using a V100 GPU, while our experiments are run on an A100 GPU, which may introduce slight variations in performance.
Additionally, we incorporate two recent methods designed for scaling tabular classification: TuneTables \citep{feuer2024tunetables}, a TF method, and HyperFast \citep{bonet2024hyperfast}, a Non-TF NN method.

Note that not all baselines successfully ran on all datasets. 
Many methods face constraints and encounter issues, particularly with the TabZilla hard benchmark, often due to poor scalability. 
We explicitly indicate which methods failed to run smoothly across all datasets.
Originally, \tabpfn{} was limited to datasets with no more than 100 features and 10 classes. To facilitate a fair comparison between \ours{} and \tabpfn{}, we implemented workarounds to prevent \tabpfn{} from encountering errors. 
For datasets exceeding 100 features, we performed random feature selection. For those with more than 10 classes, we evaluated the accuracy of the nine most prevalent classes and marked all other classes as other, and incorrect.
For TuneTables, we directly import \texttt{TuneTablesClassifier} from their Python package \texttt{tunetables}.
Note that our results differ from those reported in their paper, as their study involved more extensive hyperparameter search, which significantly increased runtime.
We also compare our methods with TuneTables using the dataset split specified in their paper's setting, with results deferred to Section~\ref{app:tunetables}.
Similarly, for HyperFast, we utilize \texttt{HyperFastClassifier} directly from their Python package \texttt{hyperfast} default parameters.
Notably, HyperFast is meta-trained on many datasets we use for evaluation. %

\subsection{Results on 98 Simple Datasets (Table 1, \citet{mcelfresh2023tabzilla})}\label{App:sec_detailed_result}
The results, presented in Table~\ref{tab:tabzilla1}, are consistent with the conclusions drawn in the main text.
\begin{table*}[t]
    \centering
    \resizebox{0.75\textwidth}{!}{
\begin{tabular}{l|l|cc|cc|rr}
\toprule
 \multirow{2}{*}{\textbf{Algorithm}} & \multirow{2}{*}{\textbf{Class}} & \multicolumn{2}{c|}{\textbf{Mean AUC}} & \multicolumn{2}{c|}{\textbf{Std. AUC}} & \multicolumn{2}{c}{\textbf{Time / 1000 inst.}} \\
 \cmidrule{3-8}
 & & median & mean & mean & median & median & mean \\
\midrule
TabPFN~\citep{hollmann2023tabpfn} & \tf{} & 0.97 & 0.84 & 0.15 & 0.08 & 0.56 & 0.74 \\
CatBoost~\citep{prokhorenkova2018catboost} & GBDT & 0.97 & 0.92 & 0.15 & 0.07 & 1.95 & 20.51 \\
\rowcolor[gray]{0.9} \textbf{\ours{} (Ours)} & \tf{} & 0.96 & 0.90 & 0.15 & 0.08 & 0.22 & 0.37 \\
XGBoost~\citep{chen2016xgboost} & GBDT & 0.96 & 0.91 & 0.16 & 0.09 & 0.38 & 0.85 \\
RandomForest~\citep{liaw2002classification} & Classical & 0.95 & 0.90 & 0.16 & 0.09 & 0.32 & 0.47 \\
SAINT~\citep{somepalli2021saint} & \tf{} & 0.94 & 0.86 & 0.16 & 0.11 & 146.15 & 170.56 \\
HyperFast~\citep{bonet2024hyperfast} & \nontf{} & 0.94 & 0.87 & 0.15 & 0.09 & 53.45 & 89.75 \\
LightGBM~\citep{ke2017lightgbm} & GBDT & 0.93 & 0.85 & 0.18 & 0.09 & 0.29 & 0.90 \\
ResNet~\citep{resnet} & \nontf{} & 0.93 & 0.85 & 0.16 & 0.10 & 8.83 & 15.99 \\
DANet~\citep{chen2022danets} & \nontf{} & 0.92 & 0.85 & 0.16 & 0.08 & 57.18 & 64.29 \\
NODE~\citep{popov2019neural} & \nontf{} & 0.91 & 0.83 & 0.16 & 0.11 & 131.73 & 160.76 \\
FTTransformer~\citep{gorishniy2021revisiting} & \tf{} & 0.89 & 0.81 & 0.17 & 0.11 & 18.04 & 27.91 \\
SVM~\citep{cortes1995support} & Classical & 0.89 & 0.78 & 0.19 & 0.09 & 2.06 & 61.18 \\
MLP-rtdl~\citep{gorishniy2021revisiting} & \nontf{} & 0.88 & 0.75 & 0.18 & 0.11 & 7.09 & 15.21 \\
DeepFM~\citep{guo2017deepfm} & \nontf{} & 0.87 & 0.77 & 0.19 & 0.12 & 4.89 & 6.05 \\
TabNet~\citep{arik2021tabnet} & \tf{} & 0.85 & 0.68 & 0.26 & 0.14 & 29.34 & 35.12 \\
STG~\citep{yamada2020feature} & \nontf{} & 0.82 & 0.71 & 0.20 & 0.14 & 15.98 & 18.58 \\
TuneTables~\citep{feuer2024tunetables} & \tf{} & 0.81 & 0.70 & 0.25 & 0.16 & 32.96 & 73.40 \\
LinearModel~\citep{cox1958regression} & Classical & 0.78 & 0.67 & 0.19 & 0.14 & 0.03 & 0.04 \\
MLP~\citep{mlp} & \nontf{} & 0.76 & 0.68 & 0.20 & 0.13 & 11.23 & 18.31 \\
DecisionTree~\citep{quinlan1986induction} & Classical & 0.74 & 0.63 & 0.24 & 0.18 & 0.01 & 0.03 \\
TabTransformer~\citep{huang2020tabtransformer} & \tf{} & 0.72 & 0.61 & 0.17 & 0.13 & 13.45 & 22.05 \\
KNN~\citep{cover1967nearest} & Classical & 0.70 & 0.61 & 0.21 & 0.14 & 0.03 & 0.05 \\
VIME~\citep{yoon2020vime} & \nontf{} & 0.60 & 0.54 & 0.25 & 0.15 & 15.60 & 17.98 \\
NAM~\citep{agarwal2021neural} & \nontf{} & 0.39 & 0.44 & 0.27 & 0.19 & 97.99 & 233.77 \\
\bottomrule
\end{tabular}}
    \captionsetup{skip=5pt}
    \caption{
    \textbf{Performance comparison of algorithms across 98 simple datasets (as used in Table 1 of \citet{mcelfresh2023tabzilla})}.
    The reported AUC values are normalized. 
    The ``Time/1000 inst.'' column represents the combined training and test time for all datasets, divided by the total number of samples.
    Notably, \ours{} achieves top 3 performance, with faster runtimes compared to baselines of similar performance, and a 2$\times$ speedup relative to \tabpfn{}.
}
    \label{tab:tabzilla1}
\end{table*}

\subsection{Evaluation on Additional Datasets}\label{app:large}
We provide additional evaluation of \ours{} on eight large datasets randomly selected from OpenML-CC18 Benchmarks~\citep{openmlcc}, after excluding the datasets contained in TabZilla's evaluation.
As shown in Table~\ref{tab:large}, \ours{} consistently outperforms \tabpfn{} in terms of speed and achieves superior performance on the majority of the datasets.

\begin{table}[htbp]
    \centering
    \resizebox{\textwidth}{!}{
\begin{tabular}{l|rrr|cc|cc}
\toprule
\multirow{2}{*}{\textbf{Dataset}} & \multirow{2}{*}{\textbf{\#Features}} & \multirow{2}{*}{\textbf{\#Instances}} & \multirow{2}{*}{\textbf{\#Classes}}  & \multicolumn{2}{c|}{\textbf{Mean AUC}} & \multicolumn{2}{c}{\textbf{Mean Time (seconds)}} \\ \cmidrule{5-8}
& & & & \tabpfn{} & \ours{} & \tabpfn{} & \ours{} \\
\midrule
kick & 33 & 72983 & 2 & 0.663 & \textbf{0.684} & 13.330 & \textbf{3.096} \\
Click-prediction-small-1220 & 10 & 39948 & 2 & 0.652 & \textbf{0.659} & 3.663 & \textbf{0.887} \\
house-8L & 9 & 22784 & 2 & \textbf{0.947} & 0.945 & 1.383 & \textbf{0.536} \\
okcupid-stem & 20 & 50789 & 3 & 0.825 & \textbf{0.828} & 6.152 & \textbf{1.511} \\
volcanoes-b1 & 4 & 10176 & 5 & 0.660 & \textbf{0.663} & 0.349 & \textbf{0.202} \\
volcanoes-b2 & 4 & 10668 & 5 & 0.651 & \textbf{0.652} & 0.375 & \textbf{0.217} \\
kdd-internet-usage & 69 & 10108 & 2 & \textbf{0.932} & \textbf{0.932} & 1.021 & \textbf{0.851} \\
BNG(tic-tac-toe) & 10 & 39366 & 2 & \textbf{0.836} & 0.835 & 3.626 & \textbf{1.111} \\
\bottomrule
\end{tabular}}
    \caption{
\textbf{Performance comparison between \tabpfn{} and \ours{} on an additional large dataset.}
We observe that \ours{} is consistently faster than \tabpfn{} and outperforms it on the majority of the datasets.
    }
    \label{tab:large}
\end{table}

\subsection{Additional Comparison with TuneTables}\label{app:tunetables}
As mentioned in Section~\ref{sec:exp}, the results of TuneTables presented in Table~\ref{tab:tunetables} of our main experiments use \texttt{TuneTablesClassifier}.
However, we note that the original paper reported results after 30 iterations of hyperparameter tuning.
They also applied this process to \tabpfn{}, using a different subset of datasets as training samples at each iteration. 
In Table~\ref{tab:tunetables}, we compare the performance of \ours{} without any hyperparameter tuning to the results reported in their paper.
\ours{} remains competitive, particularly when the number of samples is limited. 
While TuneTables tends to perform better with larger sample sizes due to its ability to update model parameters based on training data, \ours{} maintains comparable performance while being significantly faster.
\begin{table}[htbp]
\resizebox{0.9\textwidth}{!}{
\centering
\begin{tabular}{lr|cr|cr|cr}
\toprule
\multirow{2}{*}{\textbf{Dataset}} & \multirow{2}{*}{\textbf{Size}} &  \multicolumn{2}{c|}{\textbf{\tabpfn{}}} &  \multicolumn{2}{c|}{\textbf{TuneTables}} & \multicolumn{2}{c}{\textbf{\ours{}}}\\ \cmidrule{3-8}
 & & Acc. & Runtime (sec.) & Acc. & Runtime (sec.) & Acc. & Runtime (sec.) \\ \midrule
breast-cancer & 286 & .765 & \underline{29} & \underline{.770} & 65 & \textbf{.793} & \textbf{1}\\
heart-c & 303 & \underline{.848} & \underline{40} & \textbf{.903} & 66 & \textbf{.903} & \textbf{0}\\
ecoli & 336 & \underline{.848} & \underline{30} & .843 & 66 & \textbf{.882} & \textbf{0} \\
colic & 368 & .856 & \underline{39} & \textbf{.892} & 66 & \textbf{.892} & \textbf{0} \\
dresses-sales & 500 & .578 & \underline{41} & \textbf{.580} & 122 & \textbf{.580} & \textbf{0} \\
cylinder-bands & 540 & \underline{.800} & \underline{41} & \textbf{.846} & 82 & .796 & \textbf{0} \\
climate & 540 & \underline{.959} & \underline{59} & .951 & 97 & \textbf{.963} & \textbf{0} \\
balance-scale & 625 & .990 & \underline{29} & \underline{.995} & 55 & \textbf{1.000} & \textbf{0} \\
blood-transfusion & 748 & \underline{.801} & \underline{25} & .782 & 56 & \textbf{.840} & \textbf{0} \\
cmc & 1473 & .554 & \underline{91} & .556 & 109 & \textbf{.605} & \textbf{0} \\
kc-1 & 2109 & \underline{.862} & \underline{168} & .856 & 187 & \textbf{.867} & \textbf{0} \\
bioresponse & 3151 & \underline{.797} & \underline{638} & \textbf{.798} & 3012 & .720 & \textbf{13} \\
christine & 5418 & \underline{.742} & \underline{666} & \textbf{.755} & 3920 & .721 & \textbf{11}\\
robert & 10000 & .250 & \underline{964} & \textbf{.414} & 2397 & \underline{.333} & \textbf{17} \\
dilbert & 10000 & \underline{.922} & \underline{761} & \textbf{.992} & 3749 & .802 & \textbf{17} \\
har & 10299 & \underline{.936} & \underline{370} & \textbf{.981} & 2657 & .918 & \textbf{9} \\
eeg-eye-state & 14980 & \underline{.940} & \underline{178} & \textbf{.986} & 1929 & .837 & \textbf{1} \\
elevators & 16599 & \underline{.902} & \underline{186} & \underline{.902} & 1297 & \textbf{.907} & \textbf{1} \\
riccardo & 20000 & \underline{.922} & \underline{1395} & \textbf{.995} & 5247 & .773 & \textbf{31} \\
volkert & 58310 & \underline{.567} & \underline{459} & .693 & 6331 & .561 & \textbf{12} \\
higgs & 67557 & .671 & \underline{931} & \textbf{.714} & 4084 & \underline{.691} & \textbf{1}\\
connect-4 & 98050 & .668 & \underline{931} & \textbf{.817} & 5395 & \underline{.692} & \textbf{1}\\
BNG (vote) & 131072 & .968 & \underline{1976} & \textbf{.974} & 2493 & \textbf{.974} & \textbf{1} \\
albert & 425240 & \underline{.642} & \underline{2363} & .658 & 17518 & .637 & \textbf{1} \\
airlines & 539383 & \underline{.600} & \underline{2602} & \textbf{.653} & 44434 & .597 & \textbf{2}\\
BNG (labor) & 1000000 & .937 & \underline{5518}  & \textbf{.967} & 7717 & \underline{.950} & \textbf{8} \\
agrawall & 1000000 & \underline{.948} & \underline{5158} & \textbf{.950} & 45504 & \underline{.948} & \textbf{3} \\
poker-hand & 1025009 & .531 & \underline{2423} & \textbf{1.000} & 10471 & \underline{.542} & \textbf{15} \\
click-prediction-small & 1997410 & \underline{.833} & \underline{10421} & \textbf{.837} & 33148 & \underline{.833} & \textbf{5} \\
\bottomrule
\end{tabular}
}
\caption{
Accuracy comparison of \tabpfn{}, TuneTables, and \ours{} on test datasets from \citet{feuer2024tunetables}.
Results for \tabpfn{} and TuneTables are directly sourced from \citet{feuer2024tunetables},
where hyperparameter tuning was performed 30 times for both methods.
For \tabpfn{}, hyperparameters determine the subset of the dataset used in ICL.
\ours{} results are reported without hyperparameter tuning.
}
\label{tab:tunetables}
\end{table}

\subsection{Extending \ours{} for Image Classification}\label{app:image}
We explore the application of \ours{} to image classification tasks, comparing it against MLP and ResNet architectures. 
Our evaluation uses straightforward configurations without extensive hyperparameter optimization to maintain reasonable computational costs.
The MLP implementations include both two-layer and three-layer variants, each configured with 10 hidden neurons and trained for 70 epochs at a fixed learning rate of 0.001. 
The ResNet architecture employs 2 residual blocks with main and hidden dimension sizes of 128 and 256, respectively.
The experimental results demonstrate that \ours{} achieves remarkable efficiency gains, operating 30$\times$ faster than the MLP and 400$\times$ faster than the ResNet while maintaining competitive performance. 
This represents a significant advancement in image classification efficiency, particularly noteworthy given that previous approaches like \tabpfn{} were constrained to small, low-dimensional datasets.

We further evaluate \ours{} on an 10-way classification CIFAR-10 image dataset~\cite{krizhevsky2009learning} of 60K samples for which each sample is a color image of $32\times32$ size.
We deploy two approaches to convert images to 1D vectors.
First, we flatten the RGB channels to obtain a vector of 3072 dimensions.
Second, we utilize a pretrained ResNet-18~\cite{resnet} to obtain a semantically meaningful representation of the image as an 384-dim vector.
For each approach, we then feed the 1D vectors to \ours{} as a tabular dataset.
As shown in Table~\ref{tab:comparison}, with first approach (CIFAR10-flattened), we achieve an AUC of 0.791 within seconds of inference.
We find that reducing feature dimension to 30\% with PCA leads to $4$ times lower latency while preserving the AUC score.
The second approach (CIFAR10-embedding) significantly increases AUC to $92.2\%$ while reducing the inference latency by six times.

\begin{table*}[h!]
    \centering
     \resizebox{0.9\textwidth}{!}{
    \begin{tabular}{l|cc|cc|cc|cc}
        \toprule
        \multirow{2}{*}{\textbf{Dataset}}& \multicolumn{2}{c}{\textbf{Two-Layer MLP}} & \multicolumn{2}{c}{\textbf{Three-Layer MLP}} & \multicolumn{2}{c}{\textbf{ResNet}} & \multicolumn{2}{c}{\textbf{\ours{} (Ours)}} \\
        \cmidrule(lr){2-9}
        & \textbf{AUC} & \textbf{Time (s)} & \textbf{AUC} & \textbf{Time (s)} & \textbf{AUC} & \textbf{Time (s)} & \textbf{AUC} & \textbf{Time (s)} \\
        \midrule
        MNIST & 0.924 & 23.547 (30.5$\times$) & 0.959 & 23.060 (29.9$\times$) & - & - & 0.948 & 0.771 \\
        Fashion-MNIST & 0.793 & 23.340 (28.8$\times$) & 0.853 & 23.604 (29.1$\times$) & .990 & 398.45 (491.1$\times$) & 0.979 & 0.810 \\
        CIFAR-10 (flattened) & -& -& -& -& -& -& 0.791 & 5.872\\
        CIFAR-10 (embedding)&- &-& -&- & -& & 0.922 & 0.989\\
        \bottomrule
    \end{tabular}}
        \captionsetup{skip=5pt}
    \caption{Performance comparison of \ours{} against baseline models on image datasets.
    \textit{*Note: MLP and ResNet require significantly more time for training and inference, compared to \ours{}. Missing evaluation will be provided in the supplementary.} 
    }
    \label{tab:comparison}
\end{table*}

\section{Supplement to Section~\ref{sec:ablation}: Ablation Studies}\label{app:ablation}

\subsection{Datasets for Section~\ref{sec:exp_ab}: Incorporating Data-Efficient Techniques}
\label{App:dataset_ab}
For dimensionality reduction, we use the following datasets from OpenML: dna, musk, scene, jasmine, semeion, SpeedDating, hill-valley, mfeat-factors.
These datasets are selected from Table~\ref{tab:dataset_tabzilla1} where feature dimensions are greater than 100.
For the random sampling experiment, we use the following datasets from OpenML: cmc, kc1, car, yeast, car-evaluation, mfeat-morphological, mfeat-zernike, banknote-authentication, socmob.
The tested datasets are selected from Table~\ref{tab:dataset_tabzilla1} where the data size is greater than 1000 instances, and the feature dimension is lower than 100.

\subsection{Performance and Runtime vs. Training Sample Size}
We have demonstrated that \ours{} performs well across diverse tasks. 
Here, we provide a more fine-grained analysis, examining how performance and runtime vary with the number of training samples.
Following the setup of \tabpfn{}~\citep{hollmann2023tabpfn}, we generate synthetic datasets with sample sizes ranging from 1,000 to 12,000 and feature dimensions of 800 and 1,000. 
Results are averaged over 20 synthetic datasets and presented in Fig.~\ref{fig:tabflex_speedup}. 
We observe that accuracy consistently improves with more samples, while runtime increases linearly with the sample size, regardless of feature count.

\begin{figure}
\centering
    \begin{subfigure}[b]{.49\linewidth}
        \includegraphics[width=\linewidth]{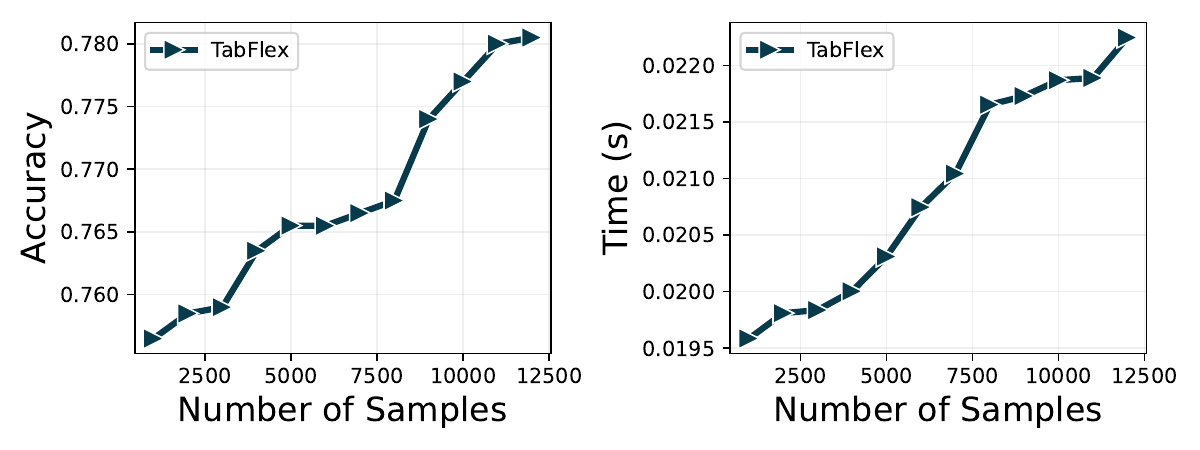}
        \caption{800 features}
    \end{subfigure}
    \begin{subfigure}[b]{.49\linewidth}
        \includegraphics[width=\linewidth]{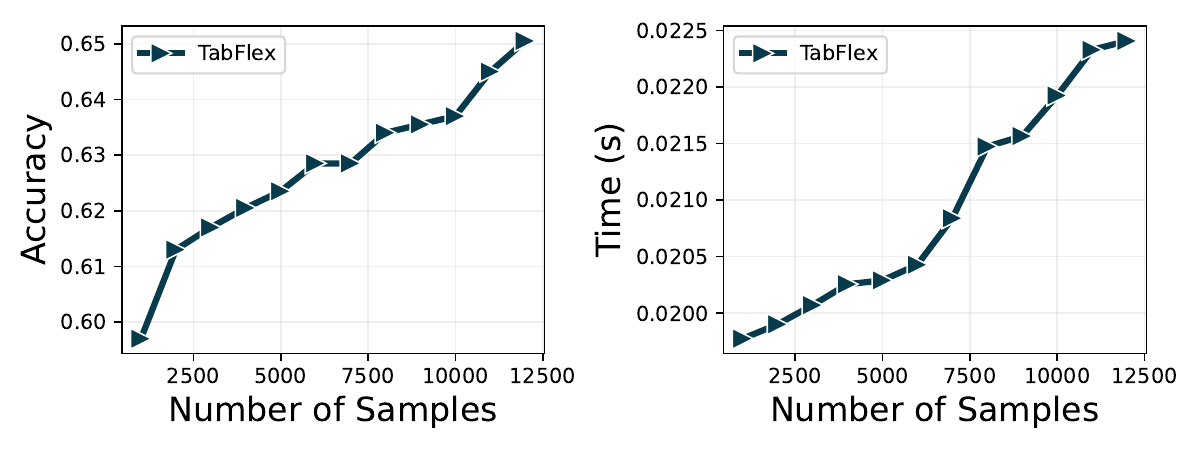}
        \caption{1000 features}
    \end{subfigure}
    \caption{
    \textbf{Accuracy and runtime versus the number of samples.} Two settings are considered: (a) 800 features and (b) 1000 features. Each curve is averaged over 20 synthetic datasets with varying data distributions, generated the same algorithm as employed in TabPFN~\citep{hollmann2023tabpfn}.
    }
    \label{fig:tabflex_speedup}
\end{figure}

\end{document}